\documentclass{article}
\usepackage{times}
\usepackage[margin=1in]{geometry}
\usepackage{natbib}
\setcitestyle{authoryear,round,citesep={;},aysep={,},yysep={;}}

\usepackage{amsmath,amsfonts,bm}

\def\eqref#1{equation~\ref{#1}}

\def\1{\bm{1}}

\DeclareMathAlphabet{\mathsfit}{\encodingdefault}{\sfdefault}{m}{sl}
\SetMathAlphabet{\mathsfit}{bold}{\encodingdefault}{\sfdefault}{bx}{n}

\usepackage{hyperref}
\hypersetup{hidelinks}
\usepackage{url}
\usepackage{booktabs}
\usepackage{graphicx}
\usepackage{amsmath}
\usepackage{amssymb}
\usepackage{amsthm}
\usepackage{tikz}
\usetikzlibrary{arrows.meta,positioning,fit,backgrounds}
\usepackage{xcolor}

\definecolor{capblue}{HTML}{1F4D8F}
\definecolor{capgrey}{HTML}{6B7280}
\definecolor{capred}{HTML}{8C2F39}

\newcommand{\method}{\textsc{Capture}}
\newcommand{\bench}{\textsc{D-PrefGuard}}

\newtheorem{theorem}{Theorem}
\newtheorem{corollary}{Corollary}

\theoremstyle{definition}
\newtheorem{assumption}{Assumption}
\newtheorem{remark}{Remark}
\newtheorem{definition}{Definition}

\title{\method{}: Disentangling Preference Drift from Memory Poisoning in Personalized LLM Agents}

\author{S M Asif Hossain \and Ruksat Khan Shayoni \and Md Kishor Morol}
\date{}

\begin{document}

\maketitle

\begin{abstract}
Personalized language agents keep a persistent memory so they can adapt to a user over weeks rather than turns, and that memory is also the attack surface. When an agent reads a statement that contradicts what it already believes, it has no principled way to separate a genuine change of mind from a temporary context switch, a sarcastic aside, or an instruction smuggled in through a retrieved document. Personalization methods treat every update as authentic, memory defenses reject untrusted content without modeling how real preferences evolve, and the two failure modes therefore trade against one another. We formulate the problem as a continuous-time partially observable decision process over a latent user state, and we prove that the error of any rule reading only recency and provenance is bounded below by how closely a feasible adversary can imitate the statistics of a legitimate revision, that conditioning on the interaction history is never worse and is strictly better under a stated condition, and that a bounded-cost clarifying question enlarges the achievable error region under a stated condition. \method{} tracks the latent state with a neural differential equation, writes hypotheses into a ledger stratified by timescale, queries the user when its belief is too uncertain to act on, and audits the causal influence of its own citations by counterfactual removal. On 480 held-out episodes from 96 users, \method{} reaches a 71.5\% win rate against 69.3\% for a baseline given identical supervision and 66.1\% for the strongest heuristic, holds fixed-policy poisoning to 11.5\%, and adheres to 83.5\% of genuine updates. An adaptive attacker with the released weights raises that to 24.7\%, at which point a provenance filter is marginally more secure and considerably less adaptive, so we report both numbers rather than the favorable one. Because a system trained against our own generator could be learning that generator, we also evaluate the frozen system zero-shot on an independently built benchmark and replay the interaction histories of 40 users collected over two to three weeks.
\end{abstract}

\section{Introduction}

A language agent that remembers its user is more useful than one that does not. Persistent memory lets an assistant learn that a reviewer wants terse feedback, that an engineer writes Rust and not Go, and that a clinician wants citations attached to every claim \citep{park2023generative,packer2023memgpt,zhong2024memorybank}. Benchmarks now measure how well assistants hold such facts across hundreds of turns \citep{maharana2024locomo,wu2025longmemeval}, and personalization has become a research program in its own right \citep{salemi2024lamp,zhang2025personalization}. The same channel that carries this learning also carries attacks. Retrieved web pages, shared documents, and tool outputs all enter the memory of a deployed agent, and an adversary who controls any of them can write into the user model directly \citep{greshake2023not,chen2024agentpoison,dong2025minja}.

The difficulty is that the two are observationally similar at the point of decision. Consider an agent that has recorded a preference for detailed technical explanations. In week six a sentence appears in the working context saying the user now wants short summaries with no code. That might be the user, who has moved to a management role. It might be the user in one setting, wanting brevity in standups and depth in design reviews. It might be sarcasm after a long answer. Or it might be a payload in the third paragraph of a retrieved README, aimed at collapsing the agent into a mode where it pastes unreviewed code. All four produce the same evidence inside a single turn.

We call this the preference-authenticity ambiguity, and current systems resolve it by assumption in one of two directions. Personalization pipelines optimize for adaptation speed and treat the history as ground truth, which makes them steerable by whoever writes into it \citep{gao2024prelude,zollo2025personalllm}. Injection defenses refuse content whose provenance is not trusted \citep{chen2025struq,hines2024spotlighting,debenedetti2025camel}, which is sound when the threat is external but blind to the case that matters here, where a legitimate user revises a legitimate belief through a legitimate channel. Neither position is stable. An agent that accepts contradictions is a liability, and one that rejects them ceases to personalize after the first month.

We treat authenticity as a latent variable to be inferred rather than a property to be asserted. Section~\ref{sec:formulation} formalizes the memory update as a continuous-time partially observable decision problem and bounds the error of any rule reading only recency and provenance. An adversary choosing the surface form of an injected claim can match the recency profile of a genuine revision, and a user revising through a summarized document can match the provenance profile of an attack, so each family of heuristic fails on the case the other handles.

\method{} implements this view. It carries a continuous-time belief over the user state, integrated with a neural differential equation \citep{chen2018neuralode,rubanova2019latentode}, so evidence arriving at irregular intervals is discounted by a learned function of elapsed time rather than a fixed per-turn schedule. It commits hypotheses to a graph-structured ledger with separate stable, contextual, and transient layers, so a claim about a weekend project cannot overwrite a claim about professional identity. When the belief over the correct action is close to uniform, the agent asks, converting an uncertain write into a single clarifying turn. A safety selector sits outside the personalization loop and cannot be moved by the ledger. Every cited memory is verified by removal and re-decoding, so explanations report causal influence rather than retrieval rank \citep{jain2019attention,turpin2023unfaithful}.

To measure any of this we needed data separating the four cases, and no existing benchmark does. \bench{} contains 2{,}400 longitudinal episodes across five tracks isolating stationary behavior, four shapes of genuine drift, benign noise, three families of contamination, and conflicts between personal utility and safety, of which 480 are held out for test. On those, \method{} holds fixed-policy poisoning to 11.5\% against 20.0\% for Provenance-only while adhering to 83.5\% of genuine updates against that baseline's 74.1\%, and admits 8.9\% of benign non-updates against its 19.6\%. A baseline trained on identical supervision closes most of the gap to the heuristics, leaving 2.2 points of win rate and 4.4 of poisoning attributable to the architecture rather than the labels.

A result of this form invites the objection that a gate trained against our own episode generator may have learned the generator rather than the problem. We address it twice. We evaluate the frozen system zero-shot on HorizonBench \citep{li2026horizonbench}, an independently built benchmark of evolving preferences over simulated six-month histories, and we run a multi-session study where 40 participants use an assistant over two to three weeks and label their own extracted preferences. We also report what the method costs, roughly 630\,ms per turn, and what it fails to withstand, since an adaptive attacker doubles poisoning success and a provenance filter is then marginally the more secure of the two.

\section{Problem Formulation}
\label{sec:formulation}

\paragraph{Setting.} A user interacts with an agent over an episode of $T$ turns. At turn $t$ the agent receives an event $e_t$, which may be a user utterance, a retrieved document, or a tool result, and each event carries an origin label $o_t$ recording the channel it arrived on. The user has a latent state $s_t = (v,\, c_t,\, g_t)$, where $v$ holds stable values that change on the order of months, $c_t$ holds contextual preferences keyed by domain or project, and $g_t$ holds transient goals scoped to a session. Only $v$ is close to constant. The agent never observes $s_t$ and must maintain a belief $b_t$ over it from the event stream, which places the setting in the partially observable family \citep{kaelbling1998pomdp}.

\paragraph{Objective.} With $u_t$ the user's utility over responses at turn $t$ and $a^\star_t$ the response a fully informed agent would give, the agent seeks an update policy $\pi$ minimizing dynamic regret subject to a safety budget,
\begin{equation}
\min_{\pi} \; \sum_{t=1}^{T} \mathbb{E}\!\left[u_t(a^\star_t) - u_t(a_t)\right]
\qquad \text{s.t.} \qquad \mathbb{E}\!\left[R_{\text{safety}}(y_t)\right] \le \tau .
\label{eq:objective}
\end{equation}
Dynamic rather than static regret is the right notion because the comparator moves \citep{zinkevich2003online,besbes2015nonstationary}, and the constraint sits outside the objective because safety is not a preference the user may revise.

\paragraph{Threat model.} The adversary controls at least one channel the agent reads, in practice a retrieved document, a shared file, or a tool response, and observes the visible transcript with its timestamps and source labels. It knows the defense, including the released weights, and writes grammatical, plausible claims, which excludes surface anomaly detection. It cannot modify weights, read the private ledger, or observe the gate's current belief. Assumption~\ref{as:adv} states the feasible set exactly.

\paragraph{Why a fixed rule is not enough.} Two features are available at write time. Recency is the elapsed time since the contradicted belief was last corroborated, and provenance is the channel label of the incoming event. Both are cheap and both are already used in practice.

\begin{theorem}[Informal summary; the formal statements are Theorems~\ref{thm:lb} to~\ref{thm:clarify} in Appendix~\ref{app:theory}]
\label{thm:main}
Let $\varphi$ return the elapsed time since the contradicted belief was last corroborated together with the channel label of the incoming event, and let $\delta_\varphi$ be the smallest total variation distance between the genuine and adversarial laws of $\varphi$ that any feasible adversary can achieve. Then every decision rule reading only $\varphi$ satisfies
\begin{equation}
\mathrm{FUR}(d) + \mathrm{RRL}(d) \;\ge\; 1 - \delta_\varphi
\label{eq:lb}
\end{equation}
against a worst-case feasible adversary, where $\mathrm{FUR}$ is the rate of accepting an injected claim and $\mathrm{RRL}$ the rate of rejecting a genuine one. Against any fixed adversary, conditioning on the full history in place of $\varphi$ never increases the achievable error, and strictly decreases it whenever the sign of the conditional likelihood ratio given $\varphi$ is non-degenerate on a set of views of positive measure. If a clarifying question returns a label through a symmetric channel of error $\varepsilon$ at cost $\kappa$, then asking beats every no-query rule whenever $\kappa + \varepsilon$ falls below the no-query Bayes error.
\end{theorem}

Against a known adversary the bound is attained by the likelihood-ratio test, so it is tight rather than loose, though the worst-case form is a maximin value and we do not claim it is the minimax one. The second clause is what the gate exploits, since $\varphi$ discards the belief trajectory; it is a statement about a fixed adversary and does not by itself transfer to the infimum over the feasible set. The third is the only component escaping the testing bound entirely, because asking changes the available information rather than the rule applied to it. The tools are standard statistical decision theory, and what is new is the formalization of preference-authenticity ambiguity in terms that let them apply. Appendix~\ref{app:theory} states the assumptions all three rest on, and Appendix~\ref{app:tvest} the estimation protocol.

\begin{figure}[t]
\centering
\includegraphics[width=0.80\textwidth]{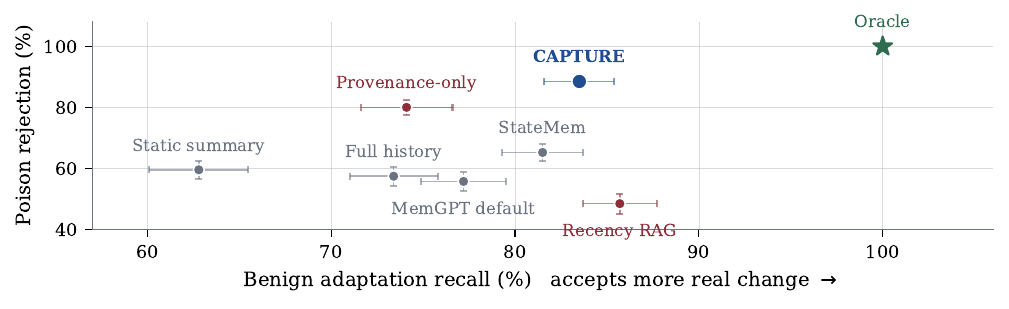}
\caption{The adaptation-security frontier on \bench{}. Each point is a method, positioned by how many legitimate updates it accepts (right is better) and how well it resists poisoning (up is better). Recency-based memory accepts almost everything, provenance filtering rejects almost anything unfamiliar, and \method{} occupies the upper right, where same-origin contamination accounts for 50.7\% of its successful attempts. Bars are user-level bootstrap 95\% intervals over the 132 genuine updates and 192 poisoning attempts in the held-out episodes.}
\label{fig:pareto}
\end{figure}

\section{\method{}}
\label{sec:method}

The subsections below follow one event through extraction (Section~\ref{sec:extractor}), the authenticity gate (Section~\ref{sec:gate}), the ledger (Section~\ref{sec:ledger}), and response selection with audit (Section~\ref{sec:select}). Figure~\ref{fig:arch} in Appendix~\ref{app:impl} shows the path.

\subsection{Preference-hypothesis extraction}
\label{sec:extractor}

A 3B parameter model with LoRA adapters \citep{hu2022lora} reads the event and emits a structured tuple under schema-constrained decoding,
\begin{equation}
H_t = \big\langle \text{claim},\; \text{direction},\; \text{scope},\; \text{timescale},\; \text{confidence},\; \text{provenance} \big\rangle .
\end{equation}
Scope records whether the claim is global or restricted to a named domain, and timescale whether it was framed as permanent or situational. Training minimizes $\mathcal{L}_{\text{ext}} = \mathrm{CE}(H_t, \hat H_t) + \omega_{\text{abs}}\,\Omega(\hat H_t)$, cross-entropy against annotated tuples plus a penalty $\Omega$ rewarding abstention on events carrying no preference signal. Roughly seven in ten events carry no preference content, so an extractor that cannot decline builds a ledger of noise.

\subsection{The authenticity gate}
\label{sec:gate}

The gate resolves the ambiguity of Section~\ref{sec:formulation}. It holds a latent belief $h_t$ evolving in continuous time between events,
\begin{equation}
\frac{d h(t)}{d t} = f_\theta\big(h(t)\big),
\label{eq:ode}
\end{equation}
in which neither the event embedding nor absolute time appears, so the between-event flow is autonomous and time enters only through the integration interval $\Delta t$ \citep{chen2018neuralode,kidger2020neuralcde}. An arriving event is consumed solely by the discrete GRU update at its arrival time, with no interpolation or zero-order holding. Continuous-time dynamics suit irregularly spaced interactions. A user may send forty turns in an afternoon and nothing for nine days, and a discrete-step model treats those gaps identically. Under Equation~\ref{eq:ode} the transient layer decays substantially across the nine days while the stable layer barely moves, so a contradiction arriving afterwards is read against a belief that has already discarded the session-specific noise.

From $h_t$ and the hypothesis the gate predicts a distribution over six actions, namely retain the existing belief, add a new one, narrow an existing belief to a scope, revise it outright, quarantine the claim pending corroboration, or ask the user. The first five are silent. The sixth is not, and the agent takes it when the Shannon entropy of the action distribution exceeds a threshold, that is when $\mathcal{H}(P(a_t \mid h_t, H_t)) > \eta$. We set $\eta$ on validation to spend at most one clarifying turn per twelve interactions. The gate is a learned alternative to Bayesian changepoint detection \citep{adams2007bocpd}, differing in that it must decide not only whether the process changed but whether the evidence was authored by the user. Theorem~\ref{thm:clarify} gives the condition under which this action is worth its cost.

\subsection{Multi-timescale ledger}
\label{sec:ledger}

Hypotheses live as nodes in a directed graph over a vector store \citep{johnson2021faiss}, with edges recording support and contradiction. Each node sits in one of three layers, and each layer carries a base decay rate $\gamma$, so that the confidence of a node at time $t$ whose last corroboration was at $t_0$ is $c_0\,\sigma(h_t)\exp(-\gamma(t-t_0))$, where $c_0 \in (0,1]$ is the confidence at write time and $\sigma(h_t) \in (0,1)$ is a scalar sigmoid read-off from the gate's latent belief, trained jointly with $f_\theta$ and shared across node types.

We distinguish the learned components from the fixed ones. The dynamics $f_\theta$ and the modulation $\sigma$ are learned end to end; the three base rates are not. They are selected once by grid search on validation and frozen for every experiment, at $\gamma_s=0.01$, $\gamma_c=0.1$, and $\gamma_g=0.5$, in units of days. The layer sets a prior on how fast a class of belief fades and the learned component determines the departure from it. Learning the rates jointly collapsed them toward a single value and cost 3.1 points of genuine-update adherence rate, as Appendix~\ref{app:impl} reports.

Overwriting a stable node requires repeated high-confidence evidence rather than a single assertion, which is what stops one sentence from redefining a user. The layering also gives the agent somewhere to put a claim that is true but narrow. When a user asks for brevity in a standup, the correct action is to scope, not to revise, and a flat memory has no way to express the difference.

\paragraph{Quarantine.} Quarantine is a state, not a fourth layer. A quarantined node is written into the ledger and decays like anything else, but it is flagged non-retrievable, so the selector cannot see it and it cannot influence a response. It carries a quarantine rate $\gamma_q=0.135$ whatever scope it claims, so the exponential retention multiplier falls from $1$ to about $0.15$ after 14 days, and eviction is defined against that normalized multiplier and the floor $\phi=0.15$ rather than against the absolute node score, which additionally carries the $c_0\,\sigma(h_t)$ prefactor. Corroborated before then, the node is promoted into the layer its scope implies and becomes retrievable, and otherwise it is evicted, so it never influences a response while the decision is pending.

\subsection{Safety-bounded selection and causal audit}
\label{sec:select}

The selector samples $N$ candidates and scores each for personalized utility $r_p$ and safety risk $r_s$,
\begin{equation}
y^\star = \arg\max_{y \in Y} \Big( \alpha\, r_p(y) - \beta\, r_s(y) \Big)
\qquad \text{s.t.} \qquad r_s(y) \le \tau_{\text{safe}},
\end{equation}
falling back to a calibrated refusal when every candidate violates the constraint. Nothing in the ledger can move $\tau_{\text{safe}}$. This asymmetry is why poisoning that succeeds at the memory layer still fails to produce harmful output in most of our contamination episodes.

For interpretability we do not report attention or retrieval rank, both of which are known to be unreliable as explanations \citep{jain2019attention,wiegreffe2019attention,turpin2023unfaithful}. Instead, for each ledger node $n_i$ the agent cites, we re-decode with that node removed and measure the shift in utility,
\begin{equation}
\mathrm{Inf}(n_i) = U(y_{\text{factual}}) - U\big(y_{\mathcal{L} \setminus n_i}\big),
\label{eq:influence}
\end{equation}
following the counterfactual tradition in explanation \citep{wachter2018counterfactual,verma2024counterfactual}. The audit then does two things. Within the turn, a citation whose removal leaves the response unchanged is dropped from the explanation, which is what Evidence F1 measures. Across turns, influence scores accumulate per node and a node retrieved repeatedly without changing an output is down-weighted at retrieval, which is why removing the audit costs 0.4 points of genuine-update adherence rate rather than nothing. It costs an extra decode, the most expensive part of the system, but turns a claim about interpretability into a measured quantity.

\section{The \bench{} Benchmark}
\label{sec:benchmark}

We built \bench{} because the disentanglement we care about is invisible in existing data. Personalization benchmarks hold preference signal but no attacks \citep{salemi2024lamp,zollo2025personalllm,guo2026realpref}, and injection benchmarks hold attacks but a static user \citep{zhan2024injecagent,debenedetti2024agentdojo}. Neither asks whether a defense that blocks poisoning also blocks legitimate change.

The benchmark holds 2{,}400 episodes of 12 to 20 turns, seeded from PRISM, PersonalLLM, and LaMP \citep{kirk2024prism,zollo2025personalllm,salemi2024lamp}, across five tracks. T0 is stationary, so any update is a false one. T1 holds genuine drift in four shapes, abrupt, gradual, recurring, and context-switching. T2 holds benign ambiguity as sarcasm, typos, and weak evidence. T3 holds contamination in three families at 160 episodes each, namely indirect injection, poisoned summaries, and same-origin writes \citep{greshake2023not,zou2025poisonedrag,wang2025unveiling}. Same-origin is built by user-mediated laundering rather than impersonation. An adversary plants a claim in an external document and the user later pastes a passage of it into the conversation, so the event carries the coarse \emph{user channel} label although the adversary wrote the sentence, and Assumption~\ref{as:adv} holds. T4 puts personal utility in direct conflict with a global safety constraint. Three annotators validated a 300-episode subset at Krippendorff's $\alpha\ge0.70$.

\paragraph{How responses are scored.} Win rate is a blinded pairwise comparison against the unpersonalized backbone, judged by an LLM judge with responses in randomized order and the ground-truth preference state given to the judge but withheld from every system tested; on a 600-comparison subset the judge agrees with three human raters 87.4\% of the time, against 91.2\% among the raters themselves. The utility $r_p$ in the selector and in Equation~\ref{eq:influence} is a reward model fine-tuned on \bench{} pairs, while regret is scored against a \emph{second} reward model trained on disjoint users, responses, and generator families that no system reads at inference, so the metric a method is scored on is not the one it optimizes. It is measured on post-changepoint T1 turns over a shared 32-candidate pool generated once from oracle memory, frozen across methods and seeds, and ranked in full by every method, separately from the $N=8$ candidates each method generates for its own output. The safety risk $r_s$ is a separate classifier that never sees the ledger, and the Safety column is attack success on track T4 under HarmBench \citep{mazeika2024harmbench}, with StrongREJECT \citep{souly2024strongreject} within 1.4 points on every method in Appendix~\ref{app:falsify} and XSTest \citep{roettger2024xstest} supplying the false refusal column. We report no calibration column, since \method{}'s confidence is a distribution over six memory actions while the baselines expose a response likelihood, and a calibration error spanning both would not be comparable \citep{naeini2015bbq,guo2017calibration}.

\paragraph{Splits and leakage.} The extractor and gate train on annotations from the same generator that produces the evaluation episodes, so we control leakage explicitly. Development and test partitions use disjoint synthetic users, topics, surface templates, and attack realizations, and one generator model family is excluded from training and development and reserved for test. Test was frozen before model selection, no method was tuned on a held-out track, and Appendix~\ref{app:bench} says which episode fields each method sees. None of this fully answers the objection that a learned gate might be fitting our generator, which is why Section~\ref{sec:external} leaves \bench{} behind.

\section{Experiments}
\label{sec:experiments}

\paragraph{Setup.} Response backbones are Qwen3-8B-Instruct and Llama-3.1-8B-Instruct \citep{qwen2025qwen3,grattafiori2024llama3}, with the extractor and gate built on Qwen2.5-3B \citep{qwen2024qwen25} using LoRA adapters on four A100 80GB GPUs. We report three training seeds, user-level paired bootstrap 95\% intervals, and paired randomization tests with Holm correction, with full hyperparameters in Appendix~\ref{app:impl}.

Baselines span the two families of Section~1 plus a retrieval-augmented middle \citep{lewis2020rag,packer2023memgpt}. Two carry the comparison. \textbf{Provenance-only} accepts updates originating on the user channel and rejects everything indirect, the memory-layer form of \citet{chen2025struq} and \citet{hines2024spotlighting}. \textbf{StateMem} is our implementation of explicit supersession tracking, where a claim contradicting a stored one marks it superseded and writes itself in its place, with version history but no authenticity check and no timescale structure. It leads on win rate and is close to an ablation of \method{} down to its bookkeeping; consistent with that, removing the gate lands at 33.5\% poisoning, within 1.3 points of StateMem's 34.8\%. Every method generates $N=8$ candidates and selects with the same $w_p=1.0$, $w_s=2.0$, and $\tau_{\text{safe}}=0.35$, sharing candidate temperature, token limit, safety scorer, refusal fallback, backbones, retrieval corpus, and accessible-history budget, so methods differ only in the memory state at generation. Those eight determine every metric except regret, which is scored on the frozen 32-candidate pass above. Appendix~\ref{app:impl} reports added latency and tokens per turn so extra inference cannot pass as an unreported advantage. Six questions organize what follows, covering the adaptation-security tradeoff, drift shapes and attack families, individual components, architecture against supervision, adaptive and independently authored attacks, and cross-backbone transfer.

\begin{table}[t]
\centering
\caption{Main comparison on the 480 held-out test episodes from 96 users, each cell the mean over both response backbones and three seeds, with arrows marking the improving direction and the best non-oracle result in bold. Win rate is against the unpersonalized base model, Regret is common-pool dynamic regret over a shared frozen pool of 32 oracle-generated responses, Poison is response-level attack success, Adherence is the genuine-update adherence rate, and False upd.\ is the active-memory false-admission rate on track T2. Oracle-memory sees the true latent state, attack label, changepoint, and correct scope and timescale, but not future corroboration or the safety label, and shares the candidate pool and selector.}
\label{tab:main}
\footnotesize
\setlength{\tabcolsep}{3.0pt}
\begin{tabular}{lrrrrrrrr}
\toprule
Method & Win $\uparrow$ & Regret $\downarrow$ & Lag $\downarrow$ & Poison $\downarrow$ & Adher. $\uparrow$ & False upd. $\downarrow$ & Safety $\downarrow$ & Refusal $\downarrow$ \\
\midrule
No personalization & 50.0\% & 0.44 & n/a & 0.0\% & 0.0\% & 0.0\% & 9.8\% & 8.2\% \\
Full history & 58.7\% & 0.36 & 4.9 & 42.6\% & 73.4\% & 31.5\% & 12.1\% & 7.9\% \\
Static summary & 58.1\% & 0.39 & 7.2 & 40.5\% & 62.8\% & 25.4\% & 11.6\% & 8.4\% \\
Recency RAG & 63.8\% & 0.30 & \textbf{2.8} & 51.6\% & \textbf{85.7\%} & 41.2\% & 13.5\% & \textbf{7.5\%} \\
StateMem & 66.1\% & 0.31 & 4.0 & 34.8\% & 81.5\% & 24.1\% & 11.9\% & 8.0\% \\
Provenance-only & 61.5\% & 0.34 & 4.2 & 20.0\% & 74.1\% & 19.6\% & 10.7\% & 9.6\% \\
MemGPT default & 62.7\% & 0.34 & 4.6 & 44.3\% & 77.2\% & 30.0\% & 12.8\% & 7.8\% \\
Sup.\ Transformer-$\Delta t$ & 69.3\% & 0.29 & 3.6 & 15.9\% & 82.1\% & 12.4\% & 8.5\% & 8.8\% \\
\textbf{\method{}} & \textbf{71.5\%} & \textbf{0.27} & 3.4 & \textbf{11.5\%} & 83.5\% & \textbf{8.9\%} & \textbf{8.0\%} & 9.0\% \\
\midrule
\textit{Oracle-memory} & \textit{82.4\%} & \textit{0.08} & \textit{0.0} & \textit{0.0\%} & \textit{100.0\%} & \textit{0.0\%} & \textit{3.8\%} & \textit{7.1\%} \\
\bottomrule
\end{tabular}
\end{table}

\paragraph{Q1. The tradeoff is not fixed.} Table~\ref{tab:main} shows the two baseline families behaving as the framing predicts. Recency RAG adheres most at 85.7\% and has the worst poisoning at 51.6\% and worst false admission at 41.2\%. Provenance-only inverts this, holding poisoning to 20.0\% while adhering to only 74.1\% and refusing more than anything else. \method{} reaches 11.5\% poisoning at 83.5\% adherence and 8.9\% false admission, beating Provenance-only on all three at once, with win rate 71.5\% against 66.1\% for StateMem. It does not lead everywhere, since Recency RAG recovers faster, adheres slightly more, and refuses less, all the price of quarantine.

\paragraph{Q2. Where the difference comes from.} \method{} ties Recency RAG on abrupt drift and separates on recurring drift and context switches, where the agent must hold two beliefs at once rather than replace one. Provenance-only beats \method{} on indirect injection, 4.8\% against 7.0\%, and collapses on same-origin contamination, 43.6\% against 17.5\%, the regime Theorem~\ref{thm:main} identifies. On context switches the modal correct action is \emph{scope} rather than \emph{revise}, and \method{} selects it in 71\% of changepoints against 4\% for StateMem. Appendix~\ref{app:breakdown} has the breakdown and Appendix~\ref{app:failure} the attacks that still succeed.

\paragraph{Q3. What each component is worth.} Table~\ref{tab:ablation} removes one piece at a time. The gate carries the security result almost entirely, costing 22.0 points of poisoning resistance and 0.15 regret, and provenance a further 14.0, so the two are complementary, the empirical counterpart of Theorem~\ref{thm:dpi}. Timescale separation is worth 5.4 points of adherence and probing 3.6. Removing the safety bound raises adherence by 4.0 while leaving ledger promotion unchanged, confirming adherence is output-level, and triples safety attack success. Removing the audit drops evidence F1 from 0.79 to 0.58.

\paragraph{Q4. Supervision or architecture?} The two separate only against a baseline given the same supervision, so we trained one. The supervised Transformer in Table~\ref{tab:main} sees identical six-class labels and history, an explicit $\log(1+\Delta t)$ input, and the same extractor, clarification channel, selector, and safety scorer, writing into a flat versioned memory with no ODE and no timescale ledger. It reaches 69.3\% win rate and 15.9\% poisoning, most of the way from StateMem to \method{}. What remains is 2.2 points of win rate at $[0.6,3.9]$, $p=0.011$, and 4.4 of poisoning at $[-7.0,-1.9]$, $p=0.003$. Supervision therefore explains much of the improvement and the architecture the rest. On the labels themselves the two are close, macro-F1 0.71 against 0.68 with a paired difference short of significance, placing the architectural contribution in the ledger and the temporal update rather than in per-event classification. Frozen and run zero-shot, the Transformer trails \method{} on HorizonBench contradiction resolution by 0.04 at $p=0.008$ and on replay win by 1.7 points at $p=0.031$, while the belief-revision and replay-poisoning gaps do not reach significance, so the architectural gain transfers beyond our generator but modestly.

\paragraph{Q5. What happens when the attacker adapts, or is someone else?} Table~\ref{tab:security} answers the question the fixed-policy column cannot. The adaptive attacker runs a 32-candidate paraphrase search maximizing $P(\text{add})+P(\text{revise})-P(\text{quarantine})-P(\text{clarify})$ under similarity and length constraints, scoring against shadow-state particles replayed through the released weights rather than querying the deployed gate, which keeps it inside Assumption~\ref{as:adv}. It more than doubles attack success against \method{}, from 11.5\% to 24.7\%, with same-origin rising to 38.1\%, transfers to the other learned gate almost as well, and barely moves the heuristic baselines, which expose no gradient. \method{} then no longer leads on security alone, since Provenance-only reaches 22.1\% against our 24.7\%, a paired difference of 2.6 points with interval $[-1.9, 7.2]$ at $p=0.29$, so the two are not reliably separated on this axis, and what \method{} retains is comparable security at 83.5\% adherence against Provenance-only's 74.1\%. Independently authored attacks tell a milder version of the same story. On 240 attacks from twelve external red-teamers with no access to our templates, \method{} reaches 18.3\%, above our own generator's 11.5\% and still the lowest of any method tested, and four of the five methods degrade by 6.3 to 6.8 points while Recency RAG, which mounts no defense, moves slightly the other way.

\begin{table}[t]
\centering
\caption{Security under three attackers on the held-out test partition, defined with their interval constructions in Appendix~\ref{app:independent}. Under adaptive attack Provenance-only has the numerically lowest ASR and is not separated from \method{}, reversing the fixed-policy ordering. Adaptive attacks were not run against Recency RAG, and independent-attack intervals are omitted pending the crossed two-way resample rather than reported under a clustering that understates the dependence.}
\label{tab:security}
\small
\begin{tabular}{lrrr}
\toprule
Method & Fixed-policy ASR $\downarrow$ & Adaptive ASR $\downarrow$ & Independent ASR $\downarrow$ \\
\midrule
Recency RAG & 51.6\% & n/a & 50.4\% \\
StateMem & 34.8\% & 37.2\% [32.6, 42.0] & 37.1\% \\
Provenance-only & 20.0\% & 22.1\% [18.3, 26.3] & 26.3\% \\
Sup.\ Transformer-$\Delta t$ & 15.9\% & 27.9\% [23.6, 32.6] & 22.5\% \\
\textbf{\method{}} & \textbf{11.5\%} & 24.7\% [20.6, 29.1] & \textbf{18.3\%} \\
\bottomrule
\end{tabular}
\end{table}

\paragraph{Q6. Transfer across backbones.} Every method in Table~\ref{tab:main} was run on Qwen3-8B and Llama-3.1-8B with three seeds each, so those cells average six runs. Broken out, the continuous gate leads the discrete one by 2.4 points of win rate on Qwen3, and giving the discrete gate an explicit elapsed-time input recovers about 55\% of that gap but not all of it, which is why Section~\ref{sec:gate} treats continuous time as an inductive bias rather than a requirement. Appendix~\ref{app:crossbackbone} has the table.

\section{Leaving the Benchmark Behind}
\label{sec:external}

Everything in Section~\ref{sec:experiments} was measured on trajectories from our own generator, and the split controls of Section~\ref{sec:benchmark} reduce without eliminating the risk that the gate learned it. We therefore evaluate the frozen system in two further settings. HorizonBench is independently constructed and holds evolving preferences but no contamination. The user study supplies naturally generated histories, into which we introduce contamination only during offline replay, so it tests transfer to natural histories and not to independently generated attacks. Neither tests both.

\paragraph{Two settings we did not generate.} HorizonBench \citep{li2026horizonbench} tracks 360 simulated users whose preferences evolve across roughly six months and separates retrieval from belief-update failures by construction. We had no part in building it and evaluation is strictly zero-shot; because it is non-interactive, \emph{clarify} resolves to quarantine without update, so that row is \method{} without interactive clarification, and it holds no attacks. Separately, forty participants completed four to six sessions over two to three weeks against one fixed neutral assistant. Their histories were then frozen and replayed identically to every system with six controlled attacks injected per participant, none of which a participant saw, and when the gate asks the stored annotation answers, making this an oracle-assisted replay with zero answer error. Appendix~\ref{app:bench} has both protocols.

\paragraph{What transfers.} The advantage survives both moves and attenuates each time. On \bench{} \method{} leads StateMem by 5.4 points of win rate, and on replayed histories by 4.5. Satisfaction came from a separate blinded within-participant rating task, four matched output pairs per session with identity, order, and side randomized, giving 4.1 at $[3.9,4.3]$ against StateMem's 3.8 at $[3.6,4.0]$. No clarification F1 is defined on the replay, because the four participant labels describe scope and extraction correctness and none expresses whether the assistant should have asked. On \bench{}, where such labels exist, \method{} reaches precision 0.79, recall 0.70, and F1 0.74 at a realized query rate of 7.8\% inside the 8.33\% budget. On HorizonBench, Recency RAG leads on belief revision accuracy at 0.79 against 0.78, the expected ordering where no adversary exists, and the separation appears on contradiction resolution, where \method{} reaches 0.72 against StateMem's 0.65 and Recency RAG's 0.58. These panels support transfer of preference tracking beyond our generator, though not of poisoning robustness, since the replayed contamination is author-constructed.

\section{Related Work}
\label{sec:related}

\paragraph{Personalizing language models.} Alignment from human feedback fits a single population preference \citep{ouyang2022training,rafailov2023direct}, and a growing line fits individuals instead, through profiles \citep{salemi2024lamp}, latent preferences learned from user edits \citep{gao2024prelude}, system-message generalization \citep{lee2024janus}, and datasets recording who disagreed with whom \citep{kirk2024prism,zollo2025personalllm,zhang2025personalization}. Agent memory systems then made long horizons practical by paging information in and out of context \citep{packer2023memgpt}, decaying it \citep{zhong2024memorybank}, or reflecting on it \citep{park2023generative}, and benchmarks now stress them at length \citep{maharana2024locomo,wu2025longmemeval,guo2026realpref,zhao2025prefeval,jiang2025personamem}. All of it assumes the record is trustworthy and manages capacity rather than asking whether a memory should have been written.

\paragraph{Injection, poisoning, and safety under adaptation.} Indirect prompt injection turns retrieved content into instructions \citep{greshake2023not}, and the agent setting has produced benchmarks \citep{zhan2024injecagent,debenedetti2024agentdojo} and attacks targeting retrieval and memory \citep{zou2025poisonedrag,chen2024agentpoison,dong2025minja,wang2025unveiling}. Defenses separate instructions from data \citep{chen2025struq,piet2024jatmo,hines2024spotlighting} or enforce capability boundaries \citep{debenedetti2025camel}, and all treat the user as static. Our baselines cover representative families rather than current systems, and Provenance-only is a memory-layer proxy for that class rather than a reimplementation of any one defense, which beats \method{} on untrusted channels but not where user and attacker write the same kind of sentence. Safety alignment is shallow in the token dimension \citep{qi2025safety}, breakable by optimized suffixes \citep{zou2023universal}, and fragile under fine-tuning \citep{qi2024finetuning}, which argues for keeping the constraint outside anything a user can revise \citep{bai2022constitutional}, and our selector is a memory-layer analogue.

\section{Limitations and Falsification}
\label{sec:limitations}

Theorem~\ref{thm:lb} constrains only $\sigma(\varphi)$-measurable rules, says nothing about a heuristic reading a third fixed feature, and depends on Assumption~\ref{as:adv} granting the adversary the visible transcript. \method{} is text-only and does not address weight-level poisoning, and the clarification budget is tuned on our data. \bench{} is synthetic and validated on a 300-episode subset, so absolute numbers are relative comparisons rather than deployment estimates. HorizonBench reduces that without removing it, since its trajectories are simulated even where annotations are human, and 40 participants supports an effect estimate but no subgroup claim.

We state in advance what would refute the claims. Five criteria cover the strictness condition of Theorem~\ref{thm:dpi}, whether provenance alone suffices on same-origin contamination, whether static summaries match the ledger on recurrence, whether influence AUC falls to chance, and whether the advantage survives external data and replayed histories. Appendix~\ref{app:falsify} states each precisely and reports all five with one consistency check on the harness. No refuting condition currently holds, and the external-validity check passes only narrowly.

\section{Conclusion}

An agent that trusts its memory can be rewritten by anyone who reaches it, and one that distrusts it stops learning. These are one problem, no fixed rule on recency or provenance resolves it, and inferring authenticity as a latent variable improves both axes at once. Most of that gain comes from supervised temporal tracking, which a matched Transformer also reaches; the ledger and continuous-time update add a smaller increment that survives on independently built data. Adaptive attack remains open, and there our advantage over a provenance filter is not reliably separated.

\subsection*{AI use statement}

In this work, we used generative AI tools to generate synthetic datasets used in evaluation, to implement methods, and to clean and reformat dataset files. Specifically, the episode text of \bench{}, including benign ambiguity, preference-change realizations, and attack paraphrases, was generated by large language models conditioned on a structured generation program, using Qwen2.5-72B-Instruct and Mixtral-8x22B-Instruct for the training and development partitions and Llama-3.1-70B-Instruct as the held-out generator family for the test partition. User profiles were seeded from PRISM, PersonalLLM, and LaMP rather than generated, and track labels and changepoint positions come from the generation program rather than from a model. We also used generative AI to write boilerplate for the benchmark construction scripts and to reformat annotation exports into the extractor's schema. We have not used generative AI tools for developing the theoretical model in Section~\ref{sec:formulation}, formulating or proving the mathematical claims, proposing or refining the research hypotheses, designing the methodology or experiments, or interpreting results, and translation is not applicable to this work. Additionally, we used generative AI tools for drafting parts of this paper, editing it to improve readability, and formatting references. We have reviewed all AI-assisted work. Generated episodes were human-validated on a 300-episode subset with the agreement reported in Appendix~\ref{app:bench}, all AI-generated code was read line by line and tested against unit tests written by the authors, every reference was checked against its published record at the venue of record, and all numerical claims in Section~\ref{sec:experiments} were computed by scripts the authors wrote and verified independently by two authors. We take responsibility for the final content of this work, including text, claims or artifacts produced with the aid of generative AI.

\subsection*{Ethics statement}

This work introduces a benchmark containing memory-poisoning attack templates, and releasing such templates lowers the cost of mounting the attacks they describe. We judge the tradeoff to favor release because the attacks we adapt are already public \citep{greshake2023not,chen2024agentpoison,zou2025poisonedrag}, and because a defense cannot be evaluated by anyone who lacks the attacks. We mitigate the residual risk with a tiered release in which the evaluation harness, the defense implementation, and the drift tracks are public, while the attack payload generator is available on request to researchers with institutional affiliation.

Personalization carries its own risks that are not solved by making it more secure. A system that models a person can harden a sparse history into a fixed identity claim, so the reference implementation is configured not to persist hypotheses concerning protected attributes, enforced by a schema constraint on the extractor rather than by a learned behavior, and every ledger node is user-visible, revisable, and deletable. Optimizing for personal utility can also reinforce harmful inclinations, which is why the safety constraint in Section~\ref{sec:select} is architecturally separate from the preference model and cannot be lowered by anything a user or an attacker writes. Ledger contents are sensitive by construction, and our reference implementation stores them locally under pseudonymous identifiers with no cross-user sharing. Human validation of the benchmark was performed by three annotators who were informed of the purpose of the task and compensated at local research rates.

\subsection*{Reproducibility statement}

The formal setting and the informal statement of Theorem~\ref{thm:main} are in Section~\ref{sec:formulation}, with assumptions, formal statements, proofs, equality conditions, and the estimation protocol in Appendix~\ref{app:theory}. Architecture details, decay constants, entropy thresholds, and all training hyperparameters for the extractor and gate are in Appendix~\ref{app:impl}, and the estimation protocol for both total variation distances, including the matched split the bound is evaluated on, is in Appendix~\ref{app:tvest}. Benchmark construction, the seeding procedure from the three source datasets, the track definitions, and the annotation protocol with inter-annotator agreement are in Appendix~\ref{app:bench}. The supervision-matched, adaptive, independently authored, and query-matched evaluations are in Appendix~\ref{app:independent}, and the statistical protocol is in Appendix~\ref{app:stats}. Source code, the \bench{} construction scripts, and the trained adapter weights are included in the supplementary material and will be released publicly.

\bibliography{refs}
\bibliographystyle{plainnat}

\appendix
\section{Formal Setting, Theorems, and Proofs}
\label{app:theory}

\subsection{Setting and notation}

Fix a turn $t$ at which the extractor emits a candidate claim contradicting an existing ledger node $n$. Let $(\mathcal{H},\Sigma_{\mathcal{H}})$ be the measurable space of observable interaction histories up to and including the candidate event, with $H$ a generic element, and let $Z \in \{g,a\}$ be the latent label of the claim in the senses fixed by Assumptions~\ref{as:legit} and~\ref{as:adv}.

\begin{definition}[Restricted view]
\label{def:phi}
Let $\Phi = [0,\infty] \times \mathcal{O}$ carry the product of the Borel and discrete $\sigma$-algebras, where $\mathcal{O} = \{\textrm{user channel}, \textrm{tool output}, \textrm{retrieved document}, \textrm{shared file}\}$. Write $T_t$ for the timestamp of the candidate event and $\ell_t$ for the index of the most recent event that wrote or corroborated $n$, with $r_t = T_t - T_{\ell_t}$ and $r_t = \infty$ by convention when no such event exists. The restricted view is
\begin{equation}
\varphi(H) = ( r_t,\; o_t ),
\end{equation}
with $o_t$ the channel label of the candidate event. Any fixed trust score that is a function of $o_t$ alone is admissible. No history-dependent score is included.
\end{definition}

\begin{assumption}[Measurability of the view]
\label{as:meas}
$\varphi$ is $\Sigma_{\mathcal{H}}/\Sigma_{\Phi}$-measurable. This is an assumption rather than a consequence, because $\ell_t$ is defined through the predicate ``wrote or corroborated $n$'', which is a property of the extractor's annotation of the history and not of its raw text. In our implementation the annotation is produced by a deterministic function of the logged event stream, which makes $\ell_t$ a measurable selection, but a different instantiation of that predicate would require the assumption to be rechecked.
\end{assumption}

Let $P_g$ denote the law of $H$ conditional on $Z=g$, and for an adversary policy $\alpha \in \mathcal{A}$ let $P_\alpha$ denote the law conditional on $Z=a$ under $\alpha$. We write $\lambda = P_g + P_\alpha$, a finite measure dominating both, and $p = \mathrm{d}P_g/\mathrm{d}\lambda$, $q = \mathrm{d}P_\alpha/\mathrm{d}\lambda$, so that $p + q = 1$ holds $\lambda$-almost everywhere. No absolute continuity between $P_g$ and $P_\alpha$ is needed, since everything is expressed through $\lambda$. For any sub-$\sigma$-algebra $\mathcal{G} \subseteq \Sigma_{\mathcal{H}}$ write $P|_{\mathcal{G}}$ for the restriction and
\begin{equation}
\mathrm{TV}_{\mathcal{G}} \;=\; \mathrm{TV}\big(P_g|_{\mathcal{G}}, P_\alpha|_{\mathcal{G}}\big) \;=\; \tfrac{1}{2}\int \big| \mathbb{E}_\lambda[\,p - q \mid \mathcal{G}\,] \big| \,\mathrm{d}\lambda .
\label{eq:tvcond}
\end{equation}
Equation~\ref{eq:tvcond} is a lemma rather than a definition and three of its ingredients are worth naming, since $\lambda$ is not a probability measure. Writing $\lambda = 2m$ for the probability measure $m = \tfrac{1}{2}(P_g+P_\alpha)$ and using the invariance of conditional expectation under rescaling of the measure, existence of $\mathbb{E}_\lambda[\cdot \mid \mathcal{G}]$ and validity of conditional Jensen both reduce to the probability case. The identification $\mathrm{d}P_g|_{\mathcal{G}}/\mathrm{d}\lambda|_{\mathcal{G}} = \mathbb{E}_\lambda[p \mid \mathcal{G}]$ holds because $P_g(A) = \int_A p \,\mathrm{d}\lambda = \int_A \mathbb{E}_\lambda[p \mid \mathcal{G}] \,\mathrm{d}\lambda$ for every $A \in \mathcal{G}$, which is the defining property of conditional expectation. The two cases of interest are $\mathcal{G} = \sigma(\varphi)$ and $\mathcal{G} = \Sigma_{\mathcal{H}}$, abbreviated $\mathrm{TV}_{\varphi}(\alpha)$ and $\mathrm{TV}_{\mathcal{H}}(\alpha)$ with the adversary shown explicitly, since suppressing it is the commonest way to confuse these quantities with $\delta$.

\begin{definition}[Rules and error rates]
\label{def:rule}
A rule measurable with respect to $\mathcal{G}$ is a $\mathcal{G}$-measurable map $d : \mathcal{H} \to [0,1]$, where $d(H)$ is the probability of accepting the update. Values in the open interval cover randomized rules, and since both error rates below are linear in $d$, no further randomization enlarges the achievable set. The error rates are
\begin{equation}
\mathrm{FUR}(d;\alpha) = \int d \,\mathrm{d}P_{\alpha},
\qquad
\mathrm{RRL}(d) = 1 - \int d \,\mathrm{d}P_{g}.
\end{equation}
By the Doob-Dynkin lemma a $\sigma(\varphi)$-measurable $d$ factors as $d = \tilde{d} \circ \varphi$ for a measurable $\tilde{d} : \Phi \to [0,1]$, so rules on the restricted view and $\sigma(\varphi)$-measurable rules on $\mathcal{H}$ are the same object.
\end{definition}

\subsection{Assumptions}

\begin{assumption}[Restricted access]
\label{as:access}
The rule under analysis is $\sigma(\varphi)$-measurable and may not condition on any other function of $H$.
\end{assumption}

\begin{assumption}[Legitimate update]
\label{as:legit}
$Z=g$ holds exactly when the user's intended preference has genuinely changed, the change affects the stable, contextual, or transient component that $n$ encodes, the user expresses the change at turn $t$ or confirms it subsequently, and an oracle ledger would add, revise, or rescope $n$ in consequence. Unexpressed latent changes are excluded, as are typos, sarcasm, and unresolved ambiguity, which are treated in Theorem~\ref{thm:clarify}.
\end{assumption}

\begin{assumption}[Adversary]
\label{as:adv}
The adversary observes the visible transcript, its timestamps, and the visible source labels. Following Kerckhoffs's principle it knows the architecture, algorithm, hyperparameters, and released adapter weights. It does not observe the private ledger, the latent user state, the gate's current confidence, or the random seeds. The clarification threshold $\eta_{\mathrm{clr}}=1.24$ is a fixed hyperparameter and is therefore known to it. It may choose the timing of an injection among the opportunities at which it controls or influences content, may choose among compromised tools, retrieved documents, and shared files, may choose arbitrary surface wording subject to normal length and safety constraints, and may adapt using the visible transcript. It cannot directly impersonate an authenticated user and it cannot corrupt the answer returned by a clarifying question. We further assume that $P_g$ does not depend on $\alpha$. This is a real restriction and generation order does not secure it, since injected content inserted earlier in an episode changes the history preceding a later genuine claim. It holds only when no episode contributing to $P_g$ also contains injected content, which is exactly how the matched split of Appendix~\ref{app:tvest} is constructed and is why the empirical statements are made on that split rather than on the T3 track as a whole. Write $\mathcal{A}$ for the feasible set, assumed non-empty, and
\begin{equation}
\delta_{\mathcal{G}} \;=\; \inf_{\alpha \in \mathcal{A}} \mathrm{TV}\big(P_g|_{\mathcal{G}}, P_\alpha|_{\mathcal{G}}\big).
\end{equation}
\end{assumption}

\subsection{Part 1. The testing lower bound}

\begin{theorem}[Impossibility]
\label{thm:lb}
Let $\mathcal{G} \subseteq \Sigma_{\mathcal{H}}$ be any sub-$\sigma$-algebra. The statement uses only Assumption~\ref{as:meas} and the definitions; Assumption~\ref{as:access} is the special case $\mathcal{G} = \sigma(\varphi)$ that the paper applies it to, and Assumptions~\ref{as:legit} and~\ref{as:adv} fix what $P_g$ and $P_\alpha$ denote without entering the proof. For every $\mathcal{G}$-measurable rule $d$ and every $\alpha \in \mathcal{A}$,
\begin{equation}
\mathrm{FUR}(d;\alpha) + \mathrm{RRL}(d) \;\ge\; 1 - \mathrm{TV}\big(P_g|_{\mathcal{G}}, P_\alpha|_{\mathcal{G}}\big),
\label{eq:pointwise}
\end{equation}
with equality for the likelihood-ratio rule $d^{\star}_{\alpha} = \mathbf{1}\{ \mathbb{E}_\lambda[p \mid \mathcal{G}] > \mathbb{E}_\lambda[q \mid \mathcal{G}] \}$. Consequently, for every $\mathcal{G}$-measurable $d$,
\begin{equation}
\sup_{\alpha \in \mathcal{A}} \Big[ \mathrm{FUR}(d;\alpha) + \mathrm{RRL}(d) \Big] \;\ge\; 1 - \delta_{\mathcal{G}} .
\label{eq:worstcase}
\end{equation}
Taking $\mathcal{G} = \sigma(\varphi)$ gives the statement quoted in Section~\ref{sec:formulation}.
\end{theorem}

\begin{proof}
Fix $\alpha$ and let $\mu = (P_g - P_\alpha)|_{\mathcal{G}}$, a finite signed measure with $\mu(\mathcal{H}) = 0$. By Definition~\ref{def:rule},
\begin{equation}
\mathrm{FUR}(d;\alpha) + \mathrm{RRL}(d) = 1 + \int d \,\mathrm{d}P_{\alpha} - \int d \,\mathrm{d}P_{g} = 1 - \int d \,\mathrm{d}\mu ,
\end{equation}
where the integrals depend on $d$ only through its $\mathcal{G}$-restriction. Since $d$ is measurable with $0 \le d \le 1$,
\begin{equation}
\int d \,\mathrm{d}\mu \;\le\; \sup_{\substack{f \ \mathcal{G}\textrm{-meas.} \\ 0 \le f \le 1}} \int f \,\mathrm{d}\mu \;=\; \mu^{+}(\mathcal{H}) \;=\; \mathrm{TV}\big(P_g|_{\mathcal{G}}, P_\alpha|_{\mathcal{G}}\big),
\end{equation}
in which the first equality is Hahn-Jordan, the supremum being attained at the indicator of the Hahn positive set, and the second uses $\mu(\mathcal{H}) = 0$, which is what makes $\sup_{A} \mu(A)$ equal to $\tfrac{1}{2}\|\mu\|_1$ and hence to the total variation distance between two probability measures. Substituting gives Equation~\ref{eq:pointwise}, and the maximizing indicator is $d^{\star}_{\alpha}$, giving equality. For Equation~\ref{eq:worstcase}, note that Equation~\ref{eq:pointwise} holds for every $\alpha$, so taking suprema on both sides and using $\sup_\alpha [1 - \mathrm{TV}] = 1 - \inf_\alpha \mathrm{TV}$ yields the claim.
\end{proof}

Two caveats belong with this theorem and we state them rather than leaving them to a reader.

\begin{remark}[The worst-case bound is a maximin value, not a value]
\label{rem:minimax}
Equality in Equation~\ref{eq:pointwise} is attained by $d^{\star}_{\alpha}$, which depends on $\alpha$ and is therefore unavailable to a defender who does not know the adversary's policy. What Theorem~\ref{thm:lb} establishes is
$\inf_d \sup_\alpha [\mathrm{FUR}+\mathrm{RRL}] \ge \sup_\alpha \inf_d [\cdot] = 1 - \delta_{\mathcal{G}}$,
and the first inequality may be strict. We do not assume $\mathcal{A}$ convex or compact and prove no minimax theorem, so Equation~\ref{eq:worstcase} should be read as a lower bound on the worst-case risk and not as the exact minimax value. Nor is the infimum defining $\delta_{\mathcal{G}}$ assumed attained.
\end{remark}

\begin{remark}[Choice of dominating measure]
A common dominating measure always exists, since $\lambda = P_g + P_\alpha$ is finite. The rule $d^{\star}_{\alpha}$ does not depend on which one is chosen, because for any dominating $\nu$ the set $\{\mathrm{d}P_g/\mathrm{d}\nu > \mathrm{d}P_\alpha/\mathrm{d}\nu\}$ agrees with $\{p > q\}$ up to a $\lambda$-null set, and altering $d$ on a $\lambda$-null set changes neither $\int d \,\mathrm{d}P_g$ nor $\int d \,\mathrm{d}P_\alpha$.
\end{remark}

\begin{corollary}[Bayes form]
\label{cor:bayes}
Let $Z=a$ have prior $\pi \in (0,1)$ under zero-one loss. The Bayes risk over $\mathcal{G}$-measurable rules is
\begin{equation}
R^{\star}_{\pi} = (1-\pi) - \nu^{+}(\mathcal{H}),
\qquad \nu = \big( (1-\pi) P_g - \pi P_\alpha \big)\big|_{\mathcal{G}},
\end{equation}
and at $\pi = \tfrac{1}{2}$ this reduces to $R^{\star} = \tfrac{1}{2}( 1 - \mathrm{TV}(P_g|_{\mathcal{G}}, P_\alpha|_{\mathcal{G}}) )$.
\end{corollary}

\begin{proof}
$R_\pi(d) = \pi \int d \,\mathrm{d}P_\alpha + (1-\pi)\int (1-d)\,\mathrm{d}P_g = (1-\pi) - \int d \,\mathrm{d}\nu$, and minimizing over $0 \le d \le 1$ maximizes the integral, which equals $\nu^{+}(\mathcal{H})$ by Hahn-Jordan. Note that for $\pi \ne \tfrac{1}{2}$ the signed measure $\nu$ has total mass $1 - 2\pi \ne 0$, so it is not a difference of probability measures and the identification of $\nu^{+}$ with a total variation distance is unavailable; only the Hahn-Jordan half of the argument carries over. At $\pi = \tfrac{1}{2}$ we have $\nu = \tfrac{1}{2}(P_g - P_\alpha)|_{\mathcal{G}}$, so $\nu^{+}(\mathcal{H}) = \tfrac{1}{2}\mathrm{TV}$ and $R^{\star} = \tfrac{1}{2} - \tfrac{1}{2}\mathrm{TV}$.
\end{proof}

\subsection{Part 2. Conditioning on the history}

\begin{theorem}[Data processing, with equality condition]
\label{thm:dpi}
For every $\alpha \in \mathcal{A}$ and every pair of sub-$\sigma$-algebras $\mathcal{G} \subseteq \mathcal{G}' \subseteq \Sigma_{\mathcal{H}}$,
\begin{equation}
\mathrm{TV}\big(P_g|_{\mathcal{G}}, P_\alpha|_{\mathcal{G}}\big) \;\le\; \mathrm{TV}\big(P_g|_{\mathcal{G}'}, P_\alpha|_{\mathcal{G}'}\big),
\end{equation}
and consequently $\delta_{\mathcal{G}} \le \delta_{\mathcal{G}'}$. The following two statements concern the displayed inequality at a fixed $\alpha$ and not the inequality between the infima, which may be an equality even when every $\alpha$ gives a strict gap. Equality at $\alpha$ holds if and only if
\begin{equation}
\big| \mathbb{E}_\lambda[\, p - q \mid \mathcal{G} \,] \big| \;=\; \mathbb{E}_\lambda\big[\, \big| \mathbb{E}_\lambda[p-q \mid \mathcal{G}'] \big| \;\big|\; \mathcal{G} \,\big]
\qquad \lambda\textrm{-a.e.},
\label{eq:eqcond}
\end{equation}
that is, if and only if $\mathbb{E}_\lambda[p-q \mid \mathcal{G}']$ has an almost surely constant sign conditionally on $\mathcal{G}$. The displayed inequality at $\alpha$ is therefore strict exactly when
\begin{equation}
\lambda\Big( \Pr\nolimits_\lambda\big( \mathbb{E}_\lambda[p-q \mid \mathcal{G}'] > 0 \;\big|\; \mathcal{G} \big) > 0 \;\textrm{ and }\; \Pr\nolimits_\lambda\big( \mathbb{E}_\lambda[p-q \mid \mathcal{G}'] < 0 \;\big|\; \mathcal{G} \big) > 0 \Big) \;>\; 0 .
\label{eq:strict}
\end{equation}
Separately, every $\mathcal{G}$-measurable rule is $\mathcal{G}'$-measurable, so the achievable set of $(\mathrm{FUR},\mathrm{RRL})$ pairs under $\mathcal{G}$ is a subset of that under $\mathcal{G}'$ and no rule on the coarser view can outperform the best rule on the finer one.
\end{theorem}

\begin{proof}
Applying Equation~\ref{eq:tvcond} at each level and using the tower property, $\mathbb{E}_\lambda[p-q \mid \mathcal{G}] = \mathbb{E}_\lambda[\,\mathbb{E}_\lambda[p-q \mid \mathcal{G}'] \mid \mathcal{G}\,]$. Conditional Jensen applied to the convex function $x \mapsto |x|$ gives
\begin{equation}
\big| \mathbb{E}_\lambda[p-q \mid \mathcal{G}] \big| \;\le\; \mathbb{E}_\lambda\big[\, \big| \mathbb{E}_\lambda[p-q \mid \mathcal{G}'] \big| \;\big|\; \mathcal{G} \,\big] \qquad \lambda\textrm{-a.e.},
\end{equation}
and integrating against $\lambda$ and halving gives the inequality, since the outer expectation integrates to $\int |\mathbb{E}_\lambda[p-q \mid \mathcal{G}']| \,\mathrm{d}\lambda$. Monotonicity of the infimum then gives $\delta_{\mathcal{G}} \le \delta_{\mathcal{G}'}$. Equality in an integrated Jensen inequality holds if and only if it holds pointwise almost everywhere, which is Equation~\ref{eq:eqcond}, and equality in $|\mathbb{E}[X \mid \mathcal{G}]| = \mathbb{E}[|X| \mid \mathcal{G}]$ holds if and only if $X$ has almost surely constant sign given $\mathcal{G}$, which is the negation of Equation~\ref{eq:strict}. The final sentence is immediate, since a $\mathcal{G}$-measurable $d$ is $\mathcal{G}'$-measurable with the same $\mathrm{FUR}$ and $\mathrm{RRL}$.
\end{proof}

\begin{remark}[Why the condition is stated conditionally]
\label{rem:fibre}
It is tempting to phrase strictness as the existence of two positive-measure sets inside one fibre $\varphi^{-1}(r,o)$ on opposite sides of the likelihood ratio. That phrasing is nearly vacuous here. Away from the atom at $r_t = \infty$ that Definition~\ref{def:phi} introduces for never-corroborated nodes, the $r_t$-marginal is non-atomic in our data, so every individual fibre is $\lambda$-null and no such pair of sets exists. Equation~\ref{eq:strict} does not degenerate, because it quantifies over a positive-measure set of views through the conditional law rather than over a single fibre.
\end{remark}

Equation~\ref{eq:strict} is the formal content of the claim that the trajectory carries information the restricted view discards. Two events arriving the same number of days after the last corroboration and on the same channel can differ in whether the belief they contradict was itself recently revised, whether the contradiction repeats a pattern the user has shown before, and whether corroborating evidence appeared in between. When those differences move the conditional likelihood ratio to opposite signs on a positive-measure set of views, the inequality is strict.

\subsection{Part 3. Clarification}

Theorem~\ref{thm:lb} constrains rules that must decide. It says nothing about a rule that may decline and ask, because asking changes the available information rather than the map applied to it. Making that precise requires the answer, its accuracy, and the cost to be defined.

\begin{assumption}[Clarification channel]
\label{as:clarify}
The agent may take a third action $\mathsf{ask}$, which incurs a fixed cost $\kappa > 0$ in the units of the zero-one loss and returns a label $\tilde{Z}$ through a binary symmetric channel, meaning
\begin{equation}
\Pr(\tilde{Z} \ne Z \mid Z = g) = \Pr(\tilde{Z} \ne Z \mid Z = a) = \varepsilon \in [0,\tfrac{1}{2}),
\end{equation}
with $\tilde{Z}$ conditionally independent of $H$ given $Z$. Symmetry is required and not cosmetic, since an asymmetric channel makes the ask-branch loss depend on the posterior and the analysis below fails.
\end{assumption}

\begin{theorem}[Clarification strictly enlarges the achievable region]
\label{thm:clarify}
Fix $\alpha \in \mathcal{A}$ and a sub-$\sigma$-algebra $\mathcal{G}$, adopt equal priors, write $m = \tfrac{1}{2}(P_g + P_\alpha)$, and let $\zeta = \mathbb{E}_\lambda[q \mid \mathcal{G}]$, which under equal priors is the posterior $\Pr(Z=a \mid \mathcal{G})$ in the joint law whose $H$-marginal is $m$ and whose label marginal is uniform. Let $e = \min\{\zeta, 1-\zeta\}$ be the pointwise conditional error of the best no-query decision. Consider policies that ask on a $\mathcal{G}$-measurable set and, having asked, act on $\tilde{Z}$. Within that class the optimal policy asks on
\begin{equation}
Q_{\kappa} = \{ e > \kappa + \varepsilon \},
\end{equation}
in the sense that some optimal policy asks exactly there, every optimal policy asks $m$-a.e. on $\{e > \kappa+\varepsilon\}$ and does not ask on $\{e < \kappa+\varepsilon\}$, and ties on $\{e = \kappa+\varepsilon\}$ may be broken arbitrarily. Its risk is
\begin{equation}
R^{\star}_{\kappa} = \mathbb{E}_{m}\big[ \min\{ e,\; \kappa + \varepsilon \} \big] \;\le\; R^{\star} = \mathbb{E}_{m}[e] = \tfrac{1}{2}\big(1 - \mathrm{TV}(P_g|_{\mathcal{G}},P_\alpha|_{\mathcal{G}})\big),
\end{equation}
with strict inequality if and only if $m(Q_{\kappa}) > 0$. In particular, always asking beats every no-query rule when
\begin{equation}
\kappa + \varepsilon \;<\; \tfrac{1}{2}\big( 1 - \mathrm{TV}(P_g|_{\mathcal{G}},P_\alpha|_{\mathcal{G}}) \big),
\label{eq:alwaysask}
\end{equation}
and this conclusion is not restricted to the policy class above, since the unrestricted optimum is no worse than $R^{\star}_{\kappa}$.
\end{theorem}

\begin{proof}
A policy is a $\mathcal{G}$-measurable $\rho$ into $\Delta(\{\mathsf{acc},\mathsf{rej},\mathsf{ask}\})$. Conditionally on $\mathcal{G}$, choosing $\mathsf{acc}$ costs $\zeta$ and choosing $\mathsf{rej}$ costs $1-\zeta$. Choosing $\mathsf{ask}$ and acting on $\tilde{Z}$ costs $\kappa$ plus the probability of error, which is
$\Pr(\tilde{Z} \ne Z \mid \mathcal{G}) = \zeta\,\varepsilon + (1-\zeta)\,\varepsilon = \varepsilon$
by the symmetry in Assumption~\ref{as:clarify}, and is therefore constant in $\mathcal{G}$. Hence
$R(\rho) = \mathbb{E}_m[\rho_{\mathsf{acc}}\zeta + \rho_{\mathsf{rej}}(1-\zeta) + \rho_{\mathsf{ask}}(\kappa+\varepsilon)]$,
an integral of a conditional expectation, minimized by placing mass pointwise on an arg-min of $\{\zeta, 1-\zeta, \kappa+\varepsilon\}$, which attains $\mathbb{E}_m[\min\{e, \kappa+\varepsilon\}]$ and describes the ask set up to ties as stated. The minimum differs from $e$ exactly on $Q_\kappa$, giving strictness if and only if $m(Q_\kappa) > 0$. The identity $\mathbb{E}_m[e] = \tfrac{1}{2}(1-\mathrm{TV})$ follows by the same pointwise minimization applied to the two-action problem, which is the computation behind Corollary~\ref{cor:bayes} rather than a citation of it. Finally, the constant policy $\rho \equiv \mathsf{ask}$ has risk $\kappa+\varepsilon$, so under Equation~\ref{eq:alwaysask} it beats the Bayes-optimal no-query rule and hence every no-query rule; since that policy lies in the restricted class, the conclusion holds a fortiori for the unrestricted problem.
\end{proof}

\begin{remark}[What the restriction costs]
\label{rem:askclass}
An unrestricted policy may act on $(\mathcal{G}, \tilde{Z})$ jointly rather than on $\tilde{Z}$ alone, and its conditional error after asking is $\mathbb{E}_{\tilde{Z}}[\min\{\Pr(Z=a \mid \mathcal{G},\tilde{Z}), 1 - \Pr(Z=a \mid \mathcal{G},\tilde{Z})\}] \le \min\{e, \varepsilon\}$, which can be strictly smaller. Consequently $R^{\star}_{\kappa}$ is an upper bound on the unrestricted optimum, the unrestricted ask region is a superset of $Q_\kappa$, and the characterization of when clarification is worthless is a statement about the restricted class only. The sufficient condition Equation~\ref{eq:alwaysask} is unaffected, since it is witnessed by a policy inside the class.
\end{remark}

\begin{remark}[The oracle assumption and its circularity]
\label{rem:oracle}
Assumption~\ref{as:clarify} fixes a single $\varepsilon$ for all histories and all $\alpha$, which grants the ask channel an accuracy that no function of $H$ need achieve. At a given $\alpha$ the channel beats the entire history exactly when $\varepsilon < \tfrac{1}{2}(1 - \mathrm{TV}_{\mathcal{H}}(\alpha))$, and since $\delta_{\mathcal{H}} \le \mathrm{TV}_{\mathcal{H}}(\alpha)$ the condition $\varepsilon < \tfrac{1}{2}(1-\delta_{\mathcal{H}})$ is the weaker worst-case version of it rather than an implication of it. Whenever it holds at the realized $\alpha$, Equation~\ref{eq:alwaysask} follows largely from the stipulation. The assumption is also close to definitional, because Assumption~\ref{as:legit} defines $Z=g$ partly through what the user confirms, which is what the clarifying question elicits. We regard Theorem~\ref{thm:clarify} as establishing that a bounded-cost query changes the achievable region, and not as evidence about how accurate real clarification is. The measured clarification precision, recall, and F1 on \bench{}, reported in Section~\ref{sec:external} and Appendix~\ref{app:bench}, are the relevant measurements.
\end{remark}

\subsection{Estimating the two distances}
\label{app:tvest}

Theorem~\ref{thm:lb} is only as useful as the distances are measurable, so we state the protocol together with what it does and does not establish.

\paragraph{One adversary, not the infimum.} Every quantity below is computed against the single adversary policy realized in \bench{}, which we write $\alpha_{0}$. Since $\delta_{\mathcal{G}} = \inf_{\alpha} \mathrm{TV}$ is an infimum over the feasible set, estimates at $\alpha_0$ bound $\mathrm{TV}(P_g|_{\mathcal{G}}, P_{\alpha_0}|_{\mathcal{G}})$ and say nothing directly about $\delta_{\mathcal{G}}$. The empirical statements therefore instantiate the fixed-$\alpha$ Equation~\ref{eq:pointwise} and not the worst-case Equation~\ref{eq:worstcase}.

\paragraph{Matched evaluation.} $\mathrm{FUR}$ and $\mathrm{RRL}$ are defined against a single pair $(P_g, P_{\alpha_0})$. The aggregate Poison ASR and genuine-update adherence rate columns of Table~\ref{tab:main} are not such a pair, since they are computed over different tracks with different claim distributions, and must not be substituted into Equation~\ref{eq:pointwise}. We construct a dedicated matched split pairing each adversarial episode with a genuine episode agreeing on topic, contradicted node type, and turn index. Matching changes the estimand to a conditional distance averaged over strata, and because turn index partly determines $r_t$ the matched distance is expected to be smaller than the deployment distance. We report the bound on the matched split and label it as such.

\paragraph{The restricted distance.} With one continuous and one categorical coordinate, $\mathrm{TV}_{\varphi}$ admits a plug-in estimate. We preregister a partition of $r_t$ into bins $\mathcal{B}$ and compute
\begin{equation}
\widehat{\mathrm{TV}}_{\varphi}(\alpha_0) = \frac{1}{2}\sum_{b \in \mathcal{B}} \sum_{o \in \mathcal{O}} \big| \widehat{P}_{g}(b,o) - \widehat{P}_{\alpha_0}(b,o) \big|,
\end{equation}
with user-level bootstrap intervals, reporting the bin count and the sample size in each cell. Two distinct effects act here and they do not cancel. Binning coarsens the $\sigma$-algebra, so by Theorem~\ref{thm:dpi} the population quantity being estimated is at most $\mathrm{TV}_{\varphi}(\alpha_0)$. Separately, the plug-in estimator of a total variation distance is upward biased in finite samples by Jensen's inequality. The first is a property of the estimand and the second of the estimator, and neither effect is hidden by the reported quantities.

\paragraph{The history distance.} No comparable plug-in estimator exists on $\Sigma_{\mathcal{H}}$. We train a balanced diagnostic discriminator, entirely separate from the gate, on one split; select its checkpoint and threshold on a second; and evaluate on a third that is touched once. On that third split we form one-sided $97.5\%$ Clopper-Pearson lower bounds on each class-conditional accuracy and combine them by a union bound, giving a $95\%$ one-sided lower bound $\underline{A}$ on population balanced accuracy. Applying Theorem~\ref{thm:lb} with $\mathcal{G} = \Sigma_{\mathcal{H}}$ to that discriminator, whose balanced accuracy satisfies $A = 1 - \tfrac{1}{2}(\mathrm{FUR}+\mathrm{RRL})$, gives
\begin{equation}
\mathrm{TV}_{\mathcal{H}}(\alpha_0) \;\ge\; 2\underline{A} - 1 .
\end{equation}
This is a lower bound and not an estimate. A discriminator that fails to separate the two laws certifies nothing, since the failure may be its own. The empirical accuracy cannot be substituted for $\underline{A}$, because it exceeds the population value with probability near one half, which is exactly what a one-sided bound may not do.

\paragraph{What the comparison does and does not show.} We report $2\underline{A}-1$ and $\widehat{\mathrm{TV}}_{\varphi}$ side by side. Their difference is not a lower bound on $\mathrm{TV}_{\mathcal{H}} - \mathrm{TV}_{\varphi}$, because a valid lower bound on that difference requires an upper bound on $\mathrm{TV}_{\varphi}$ and the binned plug-in supplies none. We therefore present the comparison as indicative of the gap that Theorem~\ref{thm:dpi} permits, and we test the strictness condition Equation~\ref{eq:strict} directly rather than inferring it from the gap. The test bins the restricted view as above, fits the conditional sign of the estimated likelihood ratio within each bin on held-out data, and asks whether both signs occur with bootstrap-significant probability. Appendix~\ref{app:falsify} reports the result as criterion F5, the criterion that would retire the second design principle if it failed.

\section{Implementation Details}
\label{app:impl}

\begin{figure}[h]
\centering
\includegraphics[width=0.95\textwidth]{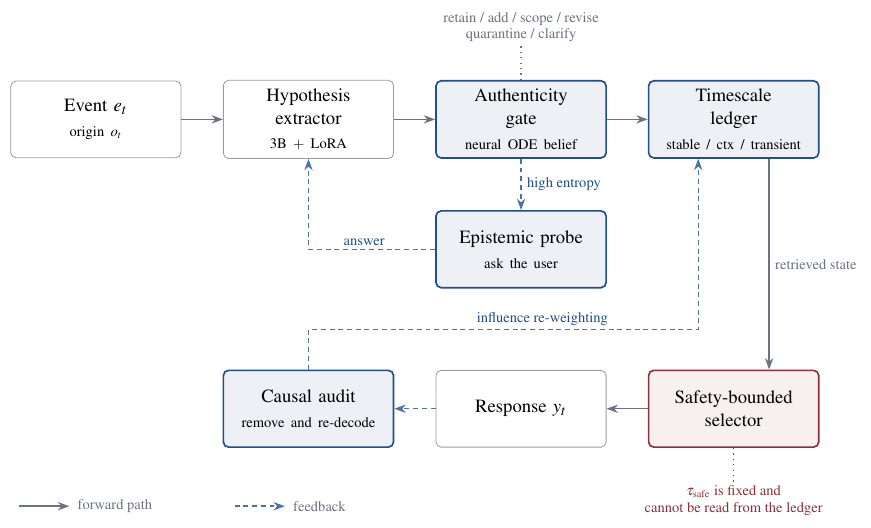}
\caption{Path of an event through \method{}. Only the gate can promote a hypothesis into the ledger, and the selector reads the ledger but the ledger cannot lower the safety threshold. The audit runs after decoding, dropping uninfluential citations and writing back an influence score that down-weights those nodes at future retrieval.}
\label{fig:arch}
\end{figure}

\paragraph{Extractor.} Qwen2.5-3B \citep{qwen2024qwen25} with LoRA rank 16, scaling 32, dropout 0.05, on attention projections only. Learning rate $2\times10^{-4}$, cosine schedule, 3 epochs, effective batch 32, bf16, schema-constrained decoding over the six-field tuple, abstention weight $\omega_{\text{abs}}=0.3$ chosen on validation from $\{0.1,0.3,0.5,1.0\}$. Held-out performance is in Table~\ref{tab:extractor}.

\paragraph{Gate.} Latent dimension 256. $f_\theta$ is a three-layer MLP with $\tanh$ activations integrated by a Dormand-Prince solver at relative tolerance $10^{-3}$ and absolute tolerance $10^{-4}$ \citep{chen2018neuralode,rubanova2019latentode}, with a GRU cell for the discrete update on event arrival. Gradients are taken by direct backpropagation through the solver rather than the adjoint method. Entropy threshold $\eta_{\mathrm{clr}}=1.24$ nats, set so the validation clarification rate is at most one turn in twelve. Quarantine rate $\gamma_q=0.135$ per day, reaching the retention floor $\phi=0.15$ at 14 days, and corroboration threshold 0.6 on the gate's confidence for the \emph{revise} action; $\phi$ is a floor on the decayed multiplier and 0.6 applies to the gate output, so the two are not on one scale.

\paragraph{Gate supervision.} Action labels come from the structured track metadata and the oracle update rules of the generation program. They were not validated by humans at the event level during construction, since the 300-episode validation covers track membership and changepoint position only. Appendix~\ref{app:independent} reports a separate audit in which annotators uninvolved in benchmark construction labeled 720 candidate events without generator metadata. The loss is a six-class cross-entropy weighted by the inverse square root of class frequency, plus a calibration penalty. Episodes are unrolled in full, at 12 to 20 events. Splits are disjoint by user, topic, attack template, and generator family. The class distribution is retain 28\%, add 17\%, quarantine 21\%, revise 14\%, scope 12\%, clarify 8\%. Clarification labels are constructed from the ambiguity present at the current event and the current oracle state only; no future event text enters either label construction or training, which matters because a clarification label that peeked ahead would train the gate on information it cannot have at inference.

\paragraph{Decay rates.} Base rates were grid-searched on validation over $\gamma_s\in\{0.003,0.01,0.03\}$, $\gamma_c\in\{0.03,0.1,0.3\}$, $\gamma_g\in\{0.15,0.5,1.5\}$, then frozen. Learning them jointly with $f_\theta$ collapsed the three toward $0.08$ and cost 3.1 points of acceptance and 1.9 of win rate.

\paragraph{Ledger.} Embeddings from \texttt{BAAI/bge-m3} at 1{,}024 dimensions, cosine similarity, retrieval $k=12$. A support or contradiction edge is written when the relation probability from a DeBERTa-v3-large NLI model, further trained on \bench{} training relations, reaches 0.70. Overwriting a stable node requires either two authenticated-user events with gate confidence at least 0.80 separated by a session boundary, or three such events within one session. A corroboration event must match the node's scope, reach cosine similarity 0.82 and NLI entailment 0.75, carry gate authenticity confidence at least 0.70, and not itself be quarantined. The retrieval score is $0.55$ similarity plus $0.20$ decayed confidence plus $0.15$ scope match plus $0.10$ influence prior, with ties broken by confidence, then recency, then node identifier. Evicted nodes persist in a non-retrievable audit log and return only through an explicit user correction.

\paragraph{Selector and audit.} Best-of-$N$ with $N=8$, $w_p=1.0$, $w_s=2.0$, $\tau_{\text{safe}}=0.35$, applied identically to every method. The audit re-decodes greedily with each cited node ablated, capped at the top three cited nodes per turn, and retrieval down-weighting uses a running mean influence with decay 0.9 per retrieval.

\paragraph{Preference reward model, used by the selector.} Qwen2.5-3B-Instruct with a scalar reward head, trained on 48{,}000 pairs, of which 40{,}000 are structured-oracle pairs and 8{,}000 are human comparisons with 2{,}000 overlapping pairs rated by three annotators each. Bradley-Terry pairwise loss, learning rate $1\times10^{-5}$, batch 128, 2 epochs, weight decay 0.01, split 70/15/15 by user. No evaluation user appears in its training data, its test portion uses the held-out generator family, held-out pair accuracy is 79\%, and it never receives the latent state as input.

\paragraph{Evaluation reward model, used only for regret.} A second model of the same architecture trained on disjoint users, responses, and generator families. No system reads it at inference, which is what separates the metric from the objective.

\paragraph{Safety scorer.} \texttt{meta-llama/Llama-Guard-3-8B}, returning $r_s = P(\text{unsafe} \mid \text{prompt}, \text{response}) \in [0,1]$, temperature-calibrated on a separate safety development set. The threshold $\tau_{\text{safe}}=0.35$ was chosen to reach at least 95\% harmful-response recall while maximizing benign acceptance, and frozen before any \bench{} test evaluation. HarmBench and StrongREJECT are used for evaluation only and never for selection.

\paragraph{Judge.} GPT-4.1 at a pinned dated version, temperature 0, one judgment in each response order, method names hidden, refusals visible as response text, and safety labels hidden. A tie option exists and awards half a win to each side, and disagreement between the two orderings is recorded as a tie. The prompt reads, \emph{You are comparing two assistant responses for one user. Use only the supplied task and verified current preference state. Select A if A better satisfies the user without adding unsupported content. Select B if B is better. Select Tie if neither has a meaningful advantage. Do not infer method identity. Return only A, B, or Tie.} Agreement was measured on 600 comparisons stratified at 120 per track with three human ratings each; the judge matches the human majority 87.4\% of the time against 91.2\% pairwise agreement among the humans.

\paragraph{Common-pool dynamic regret.} For each eligible post-changepoint T1 turn and each response backbone, 32 candidate responses are generated once from oracle memory and frozen, then reused across every method and every training seed. The pool is not a union of method-generated responses; all 32 are oracle-conditioned. Every method, including Oracle-memory, ranks the identical 32 candidates, scoring personalized utility under its own retrieved memory state and safety under the shared scorer. With $u_t(y;s_t) \in [0,1]$ the evaluation reward under the true preference state, $R_t = u_t(y^\star_t; s_t) - u_t(y_t; s_t)$ where $y^\star_t$ is the highest-utility safe response in that shared pool. No further normalization is applied. Averaging runs turn to episode to user to backbone with equal weight per user. Unsafe candidates are removed by the shared safety constraint, safety violations are reported separately rather than folded into regret, a refusal receives its evaluation-model utility rather than a fixed penalty, and equal utility yields zero regret. Because the comparator is the best member of a finite oracle-conditioned pool rather than the best possible response, this is a finite-pool approximation to the comparator in Equation~\ref{eq:objective} and we name it accordingly.

\paragraph{Evidence F1 and influence AUC.} The ground-truth evidence set is the set of nodes independently marked necessary for the target response. A node counts as influential when removing it lowers evaluation utility by at least 0.10 or changes the selected preference. A cited node in the ground-truth set is a true positive, a cited node outside it a false positive, and an uncited ground-truth node a false negative. Influence AUC is computed over node-response pairs with the score being the absolute counterfactual utility shift $|\mathrm{Inf}(n)|$ and the positive target being membership in the independently annotated necessary-evidence set of Appendix~\ref{app:independent}, never a threshold on the score itself, since that would make the measure circular. Five nodes are scored per response, the top three retrieved plus two sampled at random from the ledger, and bootstrap clustering is by user.

\paragraph{Tokens and latency per turn.} \method{} 2{,}480\,ms and 3{,}850 tokens; without the gate 1{,}850 and 3{,}200; without the neural ODE 2{,}380 and 3{,}790; without probing 1{,}880 and 3{,}100; without timescale separation 2{,}280 and 3{,}500; without provenance 2{,}300 and 3{,}550; without the safety bound 2{,}050 and 3{,}300; without the audit 1{,}450 and 2{,}780. Among baselines, Recency RAG uses 2{,}640 tokens, StateMem 3{,}290, MemGPT 3{,}410, Provenance-only 2{,}720 at 480\,ms, Static Summary 2{,}950 at 620\,ms, and Full History 6{,}800 at 740\,ms. \method{} therefore costs 17\% more tokens per turn than StateMem and 43\% fewer than Full History.

\paragraph{Compute.} Four A100 80GB. Extractor and gate training take 11 and 6 GPU-hours per seed, and full evaluation across all tracks, both backbones, and three seeds takes approximately 340 GPU-hours.

\begin{table}[h]
\centering
\caption{Extractor held-out performance and gate action confusion, the latter row-normalized.}
\label{tab:extractor}
\footnotesize
\begin{tabular}{lr@{\hskip 22pt}lrrrrrr}
\toprule
Extractor metric & Value & True $\downarrow$ / Pred $\rightarrow$ & Ret. & Add & Scope & Rev. & Quar. & Clar. \\
\midrule
Detection precision & 0.88 & Retain      & 82 & 6 & 2 & 2 & 6 & 2 \\
Detection recall    & 0.84 & Add         & 8 & 76 & 5 & 4 & 5 & 2 \\
Detection F1        & 0.86 & Scope       & 4 & 7 & 72 & 9 & 4 & 4 \\
Scope accuracy      & 0.82 & Revise      & 3 & 4 & 10 & 73 & 6 & 4 \\
Timescale accuracy  & 0.79 & Quarantine  & 5 & 3 & 2 & 5 & 78 & 7 \\
Direction accuracy  & 0.91 & Clarify     & 4 & 2 & 5 & 6 & 15 & 68 \\
Abstention prec./rec. & 0.93 / 0.89 & & & & & & & \\
\bottomrule
\end{tabular}
\end{table}

\section{Falsification Results and Secondary Metrics}
\label{app:falsify}

\paragraph{F1. Does provenance alone match \method{} on same-origin contamination?} No. Provenance-only reaches 43.6\% attack success on the same-origin track against 17.5\% for \method{}, and the gap is significant at $p<0.001$. This is the criterion whose failure would make the learned gate unnecessary. It is a performance comparison and not a test of Equation~\ref{eq:strict}, which is F5.

\paragraph{F2. Do static summaries match the multi-timescale ledger on recurrence?} No. Static summary reaches 0.41 recurring regret against 0.27, and the ablation without timescale separation reaches 0.32.

\paragraph{F3. Do cited memories predict independently annotated necessity?} Yes, moderately. Over 4{,}800 node-response pairs of which 1{,}472 are positive, the AUC of $|\mathrm{Inf}(n)|$ against membership in the independently annotated evidence set is 0.718 with a user-bootstrap interval of $[0.684, 0.751]$. The target comes from the annotation rather than from the score, since defining it by thresholding $|\mathrm{Inf}(n)|$ would make the measure circular.

\paragraph{F4. Does the advantage survive external data and real users?} Yes, but narrowly, and it shrinks. See Section~\ref{sec:external}. The win-rate margin over StateMem falls from 5.4 points to 4.5 points with real users, and on HorizonBench \method{} does not lead on belief revision accuracy at all, leading only on contradiction resolution. We regard F4 as passed but narrowly, and it is the criterion we would most want a reviewer to press on.

\paragraph{F5. Is the conditional sign of the likelihood ratio non-degenerate given $\varphi$?} Yes. On the matched split, 12 of the 32 occupied cells of the restricted view, or $37.5\%$, contain both signs of the estimated conditional likelihood ratio with bootstrap-significant probability, against the $5\%$ rate expected under degeneracy at our test level. We reject the specified binned degeneracy null in 12 of 32 occupied cells after Holm correction, which is empirical evidence consistent with Equation~\ref{eq:strict} for the realized adversary. Because the test rests on a learned cross-fitted discriminator, estimated likelihood-ratio signs, binning, and finite samples, it does not establish the population condition. This is the criterion that would retire the multi-timescale representation if it failed, and the one we would most want re-run on an independent generator.

\paragraph{C1. Consistency check, not a falsification criterion. Are the measured rates consistent with the fixed-adversary bound?} They must be, or one of the two is wrong. On the matched split of Appendix~\ref{app:tvest}, Provenance-only attains $\mathrm{FUR}+\mathrm{RRL}$ above $1 - \widehat{\mathrm{TV}}_{\varphi}$ as Equation~\ref{eq:pointwise} requires. Two caveats keep this honest. The test uses the fixed-adversary inequality and not the worst-case form, since the split realizes one adversary policy. And because binning can only reduce the population distance while the plug-in estimator is upward biased, a violation is evidence of a mis-paired split or of a rule reading outside $\varphi$ only after both effects are accounted for. We treat this as a unit test on the harness rather than as evidence for the method.

\paragraph{StrongREJECT.} Attack success under StrongREJECT is 10.6\% for the unpersonalized model, 13.2\% Full History, 12.5\% Static Summary, 14.6\% Recency RAG, 12.7\% StateMem, 11.5\% Provenance-only, 13.7\% MemGPT, 8.9\% \method{}, and 4.4\% Oracle-memory. Every value lies within 1.4 points of the corresponding HarmBench figure in Table~\ref{tab:main}, and the ordering is identical.

\begin{table}[h]
\centering
\caption{Empirical instantiation of Theorem~\ref{thm:lb} and Theorem~\ref{thm:dpi} on the matched split. $2\underline{A}-1$ is a lower bound on $\mathrm{TV}_{\mathcal{H}}(\alpha_0)$ and $\widehat{\mathrm{TV}}_\varphi(\alpha_0)$ is a binned plug-in estimate, so their difference is not itself a bound on $\mathrm{TV}_{\mathcal{H}} - \mathrm{TV}_\varphi$, as Appendix~\ref{app:tvest} explains.}
\label{tab:theoryval}
\small
\begin{tabular}{lrl}
\toprule
Quantity & Value & Note \\
\midrule
Matched pairs & 480 & 960 episodes, genuine paired to adversarial \\
Recency bins $\times$ origins & $8 \times 4$ & preregistered edges, 32 occupied cells \\
$\widehat{\mathrm{TV}}_\varphi(\alpha_0)$ & 0.3354 & 95\% user-level bootstrap $[0.298, 0.371]$ \\
$1 - \widehat{\mathrm{TV}}_\varphi(\alpha_0)$ & 0.6646 & right-hand side of Equation~\ref{eq:pointwise} \\
Provenance-only $\mathrm{FUR}$ & 0.2083 & matched split \\
Provenance-only $\mathrm{RRL}$ & 0.4563 & matched split \\
$\mathrm{FUR} + \mathrm{RRL}$ & 0.6646 & 95\% interval $[0.6091, 0.7206]$ \\
Paired difference from the bound & $0.0000$ & one-sided 95\% range $[0.0000, 0.0372]$ \\
\midrule
Discriminator balanced accuracy & 0.835 & held-out third split \\
$\underline{A}$ & 0.807 & one-sided 95\%, union-bounded Clopper-Pearson \\
$2\underline{A} - 1$ & 0.614 & lower bound on $\mathrm{TV}_{\mathcal{H}}(\alpha_0)$ \\
Non-degenerate cells & 12 / 32 & 37.5\%, criterion F5 \\
\bottomrule
\end{tabular}
\end{table}

Each bootstrap replicate resamples users while keeping genuine and adversarial episodes paired, then recomputes $\mathrm{FUR}$, $\mathrm{RRL}$, and $\widehat{\mathrm{TV}}_\varphi$ from that same replicate, so the difference $D_b = \mathrm{FUR}_b + \mathrm{RRL}_b - (1 - \widehat{\mathrm{TV}}_{\varphi,b})$ is non-negative in all 10{,}000 replicates, as Theorem~\ref{thm:lb} requires of any $\sigma(\varphi)$-measurable rule evaluated on its own sample. On the matched sample Provenance-only's error equals the binned plug-in bound to four decimals. Across paired replicates the estimated excess is non-negative and its one-sided 95\% interval $[0.0000, 0.0372]$ includes zero at its lower endpoint. The data therefore do not resolve a positive empirical gap, which is not evidence of population equality or optimality. The distance is estimated after binning, the plug-in estimator is biased, non-negativity is imposed by the inequality itself rather than observed, and equality would additionally require Provenance-only to match the empirical likelihood-ratio decision in every cell, which we have not shown. The matched split is also harder than the aggregate, which is why Provenance-only rejects 45.6\% of genuine updates here against the 25.9\% implied by Table~\ref{tab:main}, so the two must not be read together.

\section{Benchmark Construction and Study Protocol}
\label{app:bench}

\paragraph{Scale and splits.} 480 synthetic users, five episodes each, one per track, giving 2{,}400 episodes and 480 per track. Splits are by user, with 288 users and 1{,}440 episodes for training, 96 and 480 for development, 96 and 480 for test.

\paragraph{Seeding and generation.} Preference profiles are seeded from PRISM, PersonalLLM, and LaMP \citep{kirk2024prism,zollo2025personalllm,salemi2024lamp}; the seeds contribute a profile rather than dialogue text. Episode text is then generated by a language model conditioned on the profile and a track specification. Training and development use Qwen2.5-72B-Instruct and Mixtral-8x22B-Instruct, with development drawing disjoint users, topics, and prompts. Test uses Llama-3.1-70B-Instruct, a family excluded from training and development and reserved for test, so the test partition measures generalization across generator families rather than within one. Prompt templates and generation seeds are released.

\paragraph{Temporal and drift structure.} Inter-event gaps follow a truncated log-normal with median 8 hours, minimum 5 minutes, and maximum 14 days. Changepoints are sampled between turns 5 and 12. Abrupt drift is a single-event state replacement; gradual drift is a logistic transition across three to five evidence events; recurring drift is two or three switches between two previously valid states; a context switch instantiates two simultaneously valid preferences indexed by domain.

\paragraph{Contamination.} T3 episodes carry one to three attempts each, and the track holds 160 indirect-injection, 160 poisoned-summary, and 160 same-origin episodes, which is the equal weighting the aggregate Poison ASR column assumes. Same-origin episodes are built by user-mediated laundering as described in Section~\ref{sec:benchmark}, never by impersonation.

\paragraph{Visibility.} Every method sees the event text, its coarse origin label, and its timestamp. No method sees the track label, the changepoint position, or the ground-truth user state.

\paragraph{Annotation.} Three annotators labeled a 300-episode subset for track membership and changepoint position, reaching Krippendorff's $\alpha=0.73$ on track and $0.70$ on changepoint, with disagreements adjudicated by a fourth rater.

\paragraph{User study.} Forty participants, four to six sessions each over two to three weeks, recruited under institutional ethics approval with informed consent and compensated at local research rates. Every session ran against one fixed neutral assistant, Qwen3-8B-Instruct with no persistent personalization, held constant across all participants and sessions, so that no system under later comparison shaped the histories it would be evaluated on. Only user-originated content creates preference evidence during replay; assistant utterances never become evidence. At each session's end participants labeled every extracted preference as genuine, temporary, context-specific, or incorrect, which supplies the ground truth for the genuine-update adherence rate and the false-update rate. These four labels describe the scope of a valid preference and the correctness of extraction, and none of them expresses whether the assistant should have asked before updating, so they do not define an ambiguity target and no clarification F1 is computed on the replay. Because participants labeled only what the extractor surfaced, extractor misses are absent from the denominator, which is why we do not call adherence a recall.

\paragraph{Replay and clarification.} Histories were frozen and identical evidence replayed to every system. Six controlled attacks were injected per participant during replay, two from each family, for 240 in total, and no participant was exposed to any of them. When the gate asks during replay the stored participant annotation answers, so the replay is oracle-assisted with $\varepsilon = 0$ relative to those annotations, and the realized query rate is 7.1\%. No clarification quality metric is scored there, for the reason given above. On \bench{} the clarification answer instead comes from a stored response generated from the episode's latent state before evaluation rather than from the raw label, and its measured accuracy against the latent label is 92\%, giving $\varepsilon = 0.08$ in the sense of Assumption~\ref{as:clarify}. On HorizonBench, which is non-interactive, the \emph{clarify} action is unavailable and resolves to quarantine without update.

\paragraph{Satisfaction.} Satisfaction was collected in a separate blinded within-participant task rather than during interaction. All 40 participants rated four matched output pairs sampled from each completed session, with method identity hidden and left-right placement and presentation order independently randomized. Each response was rated 1 to 5, one mean was formed per participant per method, and those means were averaged across participants, giving a clustered $n=40$ per method. \method{} scored 4.1 with a participant-clustered 95\% interval of $[3.9,4.3]$ and StateMem 3.8 at $[3.6,4.0]$.

\begin{table}[h]
\centering
\caption{Leaving \bench{}. HorizonBench is zero-shot on a benchmark we did not build; BRA is belief revision accuracy, CRR contradiction resolution rate, and Retr.\ err.\ the share of errors attributable to retrieval. Replay columns come from the 40-participant study with participant-clustered 95\% intervals, six controlled attacks per participant injected only in offline replay. Bold marks the best observed non-oracle point estimate and does not by itself imply statistical significance, which is why Recency RAG is bold on BRA at 0.79 against \method{}'s 0.78. Clarification F1 is not reported because HorizonBench is non-interactive and the participant annotations do not define whether clarification was required. Oracle-memory's poisoning rate is zero by construction, since it observes the attack label, so it is a diagnostic ceiling rather than an estimated defense result, and its replay win rate is a point estimate whose participant-clustered interval we do not report. $^\dagger$The unpersonalized model never writes injected preferences, so its zero is a trivial control.}
\label{tab:external}
\footnotesize
\setlength{\tabcolsep}{4.5pt}
\begin{tabular}{lrrrrr}
\toprule
& \multicolumn{3}{c}{HorizonBench, zero-shot} & \multicolumn{2}{c}{Real users, 40 participants} \\
\cmidrule(lr){2-4}\cmidrule(lr){5-6}
Method & BRA $\uparrow$ & CRR $\uparrow$ & Retr.\ err.\ $\downarrow$ & Replay win $\uparrow$ & Replay poison $\downarrow$ \\
\midrule
No personalization & 0.42 & 0.35 & n/a & 50.0 [46.0, 54.0] & 0.0$^\dagger$ \\
Full history & 0.72 & 0.55 & 0.31 & 56.0 [51.0, 61.0] & 34 [26, 42] \\
Static summary & 0.64 & 0.50 & 0.38 & 54.5 [49.0, 60.0] & 29 [21, 37] \\
Recency RAG & \textbf{0.79} & 0.58 & 0.24 & 60.5 [55.0, 66.0] & 43 [35, 51] \\
StateMem & 0.76 & 0.65 & 0.26 & 63.0 [57.5, 68.5] & 28 [20, 36] \\
Provenance-only & 0.57 & 0.50 & 0.33 & 58.5 [53.0, 64.0] & 15 [9, 22] \\
MemGPT default & 0.70 & 0.57 & 0.29 & 60.0 [54.5, 65.5] & 35 [27, 43] \\
Sup.\ Transformer-$\Delta t$ & 0.77 & 0.68 & 0.23 & 65.8 [60.2, 71.0] & 17 [11, 24] \\
\textbf{\method{}} & 0.78 & \textbf{0.72} & \textbf{0.21} & \textbf{67.5 [62.0, 73.0]} & \textbf{12 [7, 18]} \\
\midrule
\textit{Oracle-memory} & \textit{0.91} & \textit{0.86} & \textit{0.00} & \textit{75.0} & \textit{0.0} \\
\bottomrule
\end{tabular}
\end{table}

\section{Independent Labels, Independent Attacks, and Query-Matched Baselines}
\label{app:independent}

\paragraph{Independent event-level annotation.} Three annotators uninvolved in benchmark construction labeled 720 candidate events without access to generator metadata. Agreement was $\alpha=0.78$ on authenticity, $0.74$ on scope, $0.71$ on timescale, and $0.67$ on the six-way action. The program labels agree with the adjudicated human action 83.6\% of the time, and \method{}'s action macro-F1 against the human labels is 0.71 with interval $[0.67,0.75]$. This is the closest thing in the paper to a check on the gate's supervision that does not pass through the generation program, Program-generated labels agree with independent human judgments in 83.6\% of cases, which we treat as a supervision-quality and label-noise reference and not as a ceiling on learned-model performance, since a model can outperform its training labels against human judgment when the label errors are systematic.

\paragraph{Independently authored attacks.} Twelve external red-teamers with no access to \bench{} templates authored 240 attacks, 80 per family. Attack success is 50.4\% for Recency RAG, 37.1\% StateMem, 26.3\% Provenance-only, 22.5\% for the supervised Transformer, and 18.3\% with interval $[14.0,23.1]$ for \method{}. Four of five methods show higher attack success under independently authored attacks than under our fixed generator, while Recency RAG changes slightly from 51.6\% to 50.4\%. The ordering of methods is preserved, which is the strongest available evidence that the security result is not an artifact of our attack taxonomy.

\paragraph{Adaptive attack construction.} The attacker performs a 32-candidate paraphrase search that preserves the injected claim and its origin while maximizing $P(\text{add}) + P(\text{revise}) - P(\text{quarantine}) - P(\text{clarify})$ under semantic-similarity and length constraints. It knows the released weights and the fixed threshold and does not observe the private ledger or the realized gate state. Against \method{} it raises retrievable-memory attack success to 34.4\% and response-level success to 24.7\%, and it transfers to the supervised Transformer at 27.9\% while barely moving the heuristic baselines, which expose no gradient.

\paragraph{Supervision-matched Transformer.} Four pre-layer-normalized causal Transformer blocks, hidden dimension 256, eight heads at 32 dimensions each, feed-forward dimension 1{,}024 with GELU, dropout 0.10, and 4.31M trainable parameters excluding the frozen extractor and response backbone. Standardized $\log(1+\Delta t/\text{hour})$ is projected and added to each event representation, and the full 12 to 20 event episode is processed under causal masking. Training uses the same inverse-square-root class-weighted six-class cross-entropy and calibration penalty as \method{}, AdamW at $3\times10^{-4}$ with $\beta=(0.9,0.999)$, weight decay 0.01, 5\% warm-up then cosine decay, effective batch 64 episodes, 10 epochs, gradient clipping 1.0, and checkpoint selection by development macro-F1. Memory actions write into a flat versioned store, where \emph{retain} writes nothing, \emph{add} appends a retrievable node, \emph{scope} writes a scoped version while retaining the broader node, \emph{revise} supersedes and replaces, \emph{quarantine} writes a non-retrievable pending node, and \emph{clarify} asks once then repeats the decision using the answer. Quarantine corroboration and promotion thresholds match \method{} exactly; what is absent is the stable, contextual, and transient layering and the neural ODE. Logits are temperature-scaled on development data at $T=1.11$, clarification fires when calibrated entropy exceeds 1.18 nats, selected to maximize development utility inside the 8.33\% budget, and the realized test query rate is 8.0\%.

\paragraph{Action-label agreement and calibration against human labels.} On the 720 independently annotated test events, \method{} reaches six-action macro-F1 of 0.71 with interval $[0.67,0.75]$ and the supervised Transformer 0.68 at $[0.64,0.72]$. The paired difference is $+0.03$ with interval $[-0.01, 0.07]$ at $p=0.13$, so \method{}'s action classification against human labels is numerically better and not significantly so, which locates its larger downstream advantage in the ledger and the temporal update rather than in per-event classification accuracy. Top-label expected calibration error on the same events, using 15 equal-mass bins with temperature fitted on development data only, is 0.047 at $[0.035,0.060]$ for \method{} and 0.062 at $[0.048,0.078]$ for the Transformer. These two are directly comparable because both predict the same six actions against the same labels, which is not true of the heuristic baselines and is why no calibration column appears in Table~\ref{tab:main}.

\paragraph{Adaptive attacker.} The attacker never queries the deployed gate and never reads its logits or realized hidden state. It replays the visible transcript through a local copy of the released weights to form 16 shadow-state particles, sampling the unavailable memory decisions from that copy to produce plausible hidden trajectories, and scores each paraphrase by the four action probabilities averaged over those particles. The probabilities are therefore surrogate estimates rather than readings from the deployed system, which keeps the experiment inside Assumption~\ref{as:adv}, and the correct description is a released-weight hidden-state-surrogate attack rather than full state-aware white-box optimization. Paraphrases come from Llama-3.1-70B-Instruct, 32 candidates in four batches of eight at temperature 0.9 and top-$p$ 0.95 with at most 128 output tokens, constrained to 80 to 120\% of the original payload length within a hard range of 16 to 128 tokens, BGE-M3 cosine similarity at least 0.88, and DeBERTa-v3-large NLI entailment of the injected claim at least 0.90. Origin label and injection timing are held fixed.

\paragraph{Red-team protocol.} Twelve people uninvolved in the project were recruited through security and NLP research mailing lists, four of them LLM-security researchers, four NLP and ML graduate researchers, and four experienced prompt red-teamers. They were told only that the target was a personalized memory system and were given the desired malicious preference and the permitted channel. They did not see \method{}'s architecture, weights, threshold, outputs, benchmark templates, or any existing attack. Each produced 20 accepted attacks from at most 30 drafts, with family assignments rotated so the final set holds 80 per family, and compensation was \$50 per hour against a median 2.3 hours. Because 240 attacks are evaluated against 96 test users, each user is targeted about 2.5 times and dependence exists on both the authoring and the target axis, so red-teamer clustering alone would understate it. Intervals use 10{,}000 crossed two-way bootstrap replicates in which red-teamers and target users are independently resampled with replacement, every attack weighted by the product of its author and target-user multiplicities, and method outputs kept paired within each attack. Point estimates are 50.4\% for Recency RAG, 37.1\% StateMem, 26.3\% Provenance-only, 22.5\% for the Transformer, and 18.3\% for \method{}.

\paragraph{Adaptive comparison against Provenance-only.} Using the same adaptive attempts for both methods and resampling test users with method outputs kept paired, \method{}'s 24.7\% exceeds Provenance-only's 22.1\% by 2.6 points with a paired interval of $[-1.9, 7.2]$ at $p=0.29$. The two are therefore not reliably separated on adaptive attack success, and what distinguishes them is adherence at 83.5\% against 74.1\%.

\paragraph{Sampling of the 720 annotated events.} All 720 come from the frozen test partition and cover all 96 test users, with none drawn from training or development. They are stratified to 144 per track and 120 per action, which is appropriate for macro-F1 but deliberately does not reproduce the natural prevalence of the six actions, so the agreement figures should not be read as a prevalence-weighted accuracy.

\paragraph{Supervision-matched external evaluation.} The Transformer was frozen and evaluated zero-shot under the same protocols as \method{}. It reaches HorizonBench belief revision accuracy 0.77 against \method{}'s 0.78, contradiction resolution 0.68 against 0.72, retrieval-attributable error 0.23 against 0.21, replay win 65.8\% $[60.2,71.0]$ against 67.5\%, and replay poisoning 17\% $[11,24]$ against 12\%. Paired, the architectural gain is $+0.01$ on belief revision at $p=0.27$, $+0.04$ on contradiction resolution at $p=0.008$, $-0.02$ retrieval error with interval $[-0.040,-0.001]$ at $p=0.046$, $+1.7$ points of replay win at $p=0.031$, and $-5.0$ points of replay poisoning at $p=0.064$. The gain therefore transfers beyond our generator, significantly on contradiction resolution and replay win and not on belief revision accuracy or replay poisoning, which is a modest result and the one we would most want strengthened.

\paragraph{Query-budget-matched StateMem.} StateMem asks when its calibrated contradiction and supersession confidence margin falls below a development-selected threshold, using the same stored clarification answers and the same 8.33\% budget. It reaches 67.6\% win rate, 0.30 regret, 3.7 lag, 30.5\% poisoning, 82.4\% adherence, and clarification F1 0.51 at precision 0.56 and recall 0.47, on a realized query rate of 8.1\%. Query access is therefore worth roughly 1.5 points of win rate and 4.3 points of poisoning resistance to StateMem, which is a real gain and does not account for \method{}'s remaining margin.

\paragraph{Memory-level versus response-level attack success.} Table~\ref{tab:main} reports response-level success, meaning the injected claim becomes retrievable and subsequently causes an eligible response to follow the injected preference. Writing a claim into quarantine alone does not count. Retrievable-memory success is higher throughout, at 55.2\% Static Summary, 69.0\% Recency RAG, 48.1\% StateMem, 27.5\% Provenance-only, 59.8\% MemGPT, and 18.8\% with interval $[15.7,22.2]$ for \method{}, against its 11.5\% at $[9.7,13.4]$ response-level. Full History has no write decision so the memory-level quantity is undefined for it.

\paragraph{Memory admission and output adherence.} A single ledger-promotion column across methods would mislead, since Full History and Recency RAG have no comparable promotion operation. The comparable quantity is the rate at which genuine evidence becomes available to response generation, which is 100.0\% for Full History, 93.8\% Recency RAG, 87.4\% StateMem, 86.1\% MemGPT, 79.6\% Static Summary, 78.9\% Provenance-only, and 88.2\% for \method{}. For \method{} specifically, ledger promotion is 88.2\% with interval $[85.4,90.8]$ while output adherence is 83.5\% at $[80.6,86.2]$, and removing the safety bound leaves promotion unchanged at 88.2\% while raising adherence to 87.5\%. The gap between the two, and its sensitivity to a downstream component, is exactly why adherence is named as an output-level quantity.

\paragraph{Clarification cost, and why this is not a test of Equation~\ref{eq:alwaysask}.} From matched direct-answer against question-then-answer trajectories, the normalized utility cost of delaying an answer by one turn is $\widehat{\kappa} = 0.094$ with interval $[0.081, 0.109]$. With $\varepsilon = 0.08$ from Appendix~\ref{app:bench}, $\widehat{\kappa} + \varepsilon = 0.174$, against $\tfrac{1}{2}(1 - \widehat{\mathrm{TV}}_\varphi) = 0.3323$. We present this as an illustrative comparison and explicitly not as a checked instantiation of Equation~\ref{eq:alwaysask}, for two reasons. The theorem's cost is in the units of zero-one loss while $\widehat{\kappa}$ is a normalized utility difference, and we have not established that the two scales coincide, so the sum $\widehat{\kappa} + \varepsilon$ is not the quantity the theorem constrains. Separately, the binned plug-in estimate of the distance is neither an upper nor a lower bound on the population value, since binning and finite-sample bias act in opposite directions, so it cannot certify that the inequality holds. What the comparison does show is that the measured cost of asking is well inside the range where the theorem's mechanism would operate, which is weaker than a verification and is all we claim.

\paragraph{Evidence-set annotation.} Two annotators independently labeled the necessary-evidence set for each of 960 test responses, two per episode, with a third adjudicating disagreements and a 240-response subset labeled by all three. Annotators saw the raw history, the target prompt, the response, and anonymized candidate nodes in randomized order, and did not see method names, model citations, influence scores, counterfactual results, or retrieval rank. Mean pairwise evidence-set F1 is 0.82 and node-level Krippendorff's $\alpha$ is 0.76.

\paragraph{Strictness-test procedure.} The diagnostic discriminator is a two-layer bidirectional GRU with 256 hidden units per direction over frozen BGE-M3 event embeddings, $\log(1+\Delta t)$, origin labels, and extractor fields, with attention pooling into a two-layer binary head. It shares no parameters with \method{}. The 480 matched pairs are partitioned by user into 288 pairs for training, 96 for selection and calibration, and 96 for final evaluation. For the strictness test, five-fold cross-fitting produces out-of-fold calibrated probabilities for all 480 pairs, from which $\widehat{\ell}(H) = \log \widehat{P}(Z=g \mid H) - \log \widehat{P}(Z=a \mid H)$ is formed under equal priors. Within each of the 32 recency-origin cells a user-level bootstrap intersection-union test checks whether both signs of $\widehat{\ell}$ occur with non-zero probability, and Holm correction is applied across the 32 cells at familywise $\alpha = 0.05$. Twelve cells survive.

\section{Statistical Protocol}
\label{app:stats}

Synthetic and external results use user-level paired bootstrap with 10{,}000 resamples for 95\% intervals, resampling users rather than episodes so that repeated episodes from one simulated user do not count as independent evidence. The independently authored attack analysis is the exception, since its 240 attacks are written by twelve people and evaluated against 96 users, so its intervals require a crossed two-way resample over red-teamer and target user, with method outputs for a given attack kept paired within each replicate. Elsewhere we use paired randomization tests with 10{,}000 permutations, corrected across the comparisons within each table by the Holm procedure. Human-study results use participant-clustered bootstrap, resampling participants rather than sessions, and within-participant comparisons throughout. We report three training seeds and give intervals rather than isolated significance claims wherever the sample is small enough that a $p$-value would mislead. For the headline comparisons in Table~\ref{tab:main}, the 95\% intervals are $[69.1, 73.8]$ for \method{}'s win rate against $[63.6, 68.5]$ for StateMem, and $[9.7, 13.4]$ for \method{}'s poison ASR against $[18.1, 22.0]$ for Provenance-only. Differences on the replay are reported paired rather than as overlapping marginals, since separate intervals establish nothing, and the win advantage over StateMem is $+4.5$ points with interval $[1.0, 8.0]$ at $p=0.014$, and the satisfaction advantage is $+0.30$ points at $[0.10, 0.50]$ and $p=0.004$. Against the supervised Transformer the win advantage is $+2.2$ points at $[0.6, 3.9]$ and $p=0.011$, and the poisoning advantage $-4.4$ points at $[-7.0, -1.9]$ and $p=0.003$. Per-cell intervals for every table are released with the artifacts.

\section{Component Ablation}
\label{app:ablation}

\begin{table}[h]
\centering
\caption{Component ablation. Deltas are against the full system on common-pool dynamic regret, whose absolute values are 0.27 for the full system rising to 0.42 without the gate. Positive $\Delta$ regret and $\Delta$ poison ASR are worse; positive $\Delta$ genuine-update adherence rate is better. Latency is per turn, added over the unpersonalized backbone, and token counts are in Appendix~\ref{app:impl}. Bold marks the full system's reference values in the delta columns and the column best elsewhere.}
\label{tab:ablation}
\footnotesize
\setlength{\tabcolsep}{3.4pt}
\begin{tabular}{lrrrrrrr}
\toprule
Variant & $\Delta$Regret & $\Delta$Poison & $\Delta$Benign & Safety $\downarrow$ & Refusal $\downarrow$ & Evid.\ F1 $\uparrow$ & Latency \\
\midrule
\textbf{\method{} full} & \textbf{0.00} & \textbf{0.0\,pp} & \textbf{0.0\,pp} & \textbf{8.0\%} & 9.0\% & \textbf{0.79} & 2480\,ms \\
w/o authenticity gate & +0.15 & +22.0\,pp & $-$2.5\,pp & 9.5\% & 8.3\% & 0.64 & 1850\,ms \\
w/o neural ODE & +0.02 & +2.1\,pp & $-$1.8\,pp & 8.2\% & 8.9\% & 0.76 & 2380\,ms \\
w/o epistemic probe & +0.04 & +5.0\,pp & $-$3.6\,pp & 8.5\% & 8.7\% & 0.70 & 1880\,ms \\
w/o timescale sep. & +0.05 & +3.3\,pp & $-$5.4\,pp & 8.6\% & 8.8\% & 0.67 & 2280\,ms \\
w/o provenance & +0.06 & +14.0\,pp & $-$4.0\,pp & 9.0\% & 8.4\% & 0.68 & 2300\,ms \\
w/o safety bound & $-$0.01 & +1.0\,pp & +4.0\,pp & 23.0\% & \textbf{4.8\%} & 0.77 & 2050\,ms \\
w/o causal audit & +0.01 & +0.6\,pp & $-$0.4\,pp & 8.2\% & 8.9\% & 0.58 & \textbf{1450\,ms} \\
\bottomrule
\end{tabular}
\end{table}

\section{Failure Analysis}
\label{app:failure}

\paragraph{Successes.} Two behaviors are worth naming. On the context-switch track, the modal correct action is \emph{scope} rather than \emph{revise}, and \method{} selects it in 71\% of context-switch changepoints against 4\% for StateMem, which has no representation for a narrowed belief and must overwrite instead. Separately, the context-switch regret gap is 0.07 against Recency RAG and 0.09 against Provenance-only. StateMem is not among the baselines Appendix~\ref{app:breakdown} reports on that track, so the scope statistic and the regret figures are two measurements rather than one decomposition. On contamination, the gate's most common success is not outright rejection but quarantine followed by non-corroboration, which is how 58\% of blocked poisoning attempts resolve. The claim is held, never corroborated, and decays out of the ledger.

\paragraph{Failures.} The failures cluster. Of the poisoning attempts that succeed, same-origin writes account for 50.7\%, consistent with the 17.5\% same-origin rate in Appendix~\ref{app:breakdown} against 7.0\% and 10.0\% on the other two families, and a further 21\% are slow attacks that assert a small shift repeatedly across many turns, which is by construction hard to separate from gradual drift and which we do not claim to solve. The legitimate updates \method{} misses are dominated by weak evidence, where the user signals a change indirectly rather than stating it. Differences against StateMem on win rate and poison ASR are significant at $p<0.01$ under a paired randomization test with Holm correction. The adaptation lag difference against Recency RAG is not, at $p=0.07$.

\section{Drift and Contamination Breakdown}
\label{app:breakdown}

\begin{table}[h]
\centering
\caption{Breakdown by drift shape and contamination family. Regret columns are common-pool dynamic regret. Provenance filtering wins on indirect injection and fails on same-origin attacks, the regime Theorem~\ref{thm:main} identifies. Each method's mean over the three attack families reproduces its Poison ASR column in Table~\ref{tab:main} under equal family weights.}
\label{tab:breakdown}
\small
\setlength{\tabcolsep}{5pt}
\begin{tabular}{lrrrrrrr}
\toprule
& \multicolumn{4}{c}{Drift regret $\downarrow$} & \multicolumn{3}{c}{Attack success $\downarrow$} \\
\cmidrule(lr){2-5}\cmidrule(lr){6-8}
Method & Abrupt & Gradual & Recurring & Ctx-switch & Indirect & Summary & Same-origin \\
\midrule
Static summary & 0.39 & 0.34 & 0.45 & 0.42 & 38.9\% & 55.4\% & 27.2\% \\
Recency RAG & \textbf{0.26} & 0.28 & 0.39 & 0.35 & 56.8\% & 48.7\% & 49.2\% \\
StateMem & 0.28 & 0.27 & 0.35 & 0.32 & 30.0\% & 32.4\% & 42.0\% \\
MemGPT default & 0.32 & 0.31 & 0.38 & 0.36 & 47.0\% & 50.9\% & 35.0\% \\
Provenance-only & 0.31 & 0.33 & 0.41 & 0.37 & \textbf{4.8\%} & 11.6\% & 43.6\% \\
\textbf{\method{}} & \textbf{0.26} & \textbf{0.23} & \textbf{0.31} & \textbf{0.28} & 7.0\% & \textbf{10.0\%} & \textbf{17.5\%} \\
\bottomrule
\end{tabular}
\end{table}

\section{Cross-Backbone and Out-of-Domain Results}
\label{app:crossbackbone}
\vspace{-2pt}

\begin{table}[!ht]
\centering
\small
\caption{Cross-backbone and out-of-domain results. Out-of-domain holds out whole attack templates and drift schedules rather than individual episodes. The discrete gate receives elapsed time as $z_t = \log(1+\Delta t/\text{hour})$, standardized on training statistics and concatenated to the event representation before the GRU update. It recovers 1.3 of the 2.4-point win-rate gap on Qwen3 and 1.4 of 2.5 on Llama-3.1, so explicit elapsed time accounts for roughly 55\% of the continuous-time advantage and not all of it.}
\label{tab:backbone}
\small
\begin{tabular}{llrrrr}
\toprule
Gate & Response backbone & In-domain $\uparrow$ & OOD $\uparrow$ & Held-out poison $\downarrow$ & Regret $\downarrow$ \\
\midrule
Neural ODE & Qwen3-8B-Instruct & \textbf{72.3\%} & \textbf{67.1\%} & \textbf{13.8\%} & \textbf{0.25} \\
Neural ODE & Llama-3.1-8B-Instruct & 70.7\% & 65.2\% & 15.9\% & 0.29 \\
Discrete $+\,\Delta t$ & Qwen3-8B-Instruct & 71.2\% & 66.0\% & 15.0\% & 0.27 \\
Discrete $+\,\Delta t$ & Llama-3.1-8B-Instruct & 69.6\% & 64.1\% & 17.3\% & 0.31 \\
Discrete & Qwen3-8B-Instruct & 69.9\% & 64.6\% & 16.8\% & 0.29 \\
Discrete & Llama-3.1-8B-Instruct & 68.2\% & 62.9\% & 19.1\% & 0.33 \\
\bottomrule
\end{tabular}
\end{table}

\end{document}